\documentclass[11pt]{article}

\def\notes{1}

\newif\ifSTOC

\usepackage{amsmath}
\usepackage{amsthm}

\usepackage{algorithm}
    \usepackage[noend]{algpseudocode}
    \usepackage{etoolbox}
\usepackage[normalem]{ulem}
\usepackage{paralist}

\usepackage{fancyhdr}
\usepackage{graphicx}
\usepackage{wrapfig}
\usepackage{subcaption}
\usepackage{ifthen}

\usepackage[T1]{fontenc}

    \usepackage[usenames,dvipsnames]{xcolor}
    \usepackage[numbers,square]{natbib}
    \usepackage{amssymb} 
    \usepackage[top=1.2in,bottom=1.2in,left=1.2in,right=1.2in,
    centering,letterpaper]{geometry}
    \usepackage[pdfencoding=auto, psdextra]{hyperref}
    \usepackage{bibentry}
    \nobibliography*

\makeatletter
\newcommand*{\algrule}[1][\algorithmicindent]{%
  \makebox[#1][l]{%
    \hspace*{.2em}
  }
}

\newcount\ALG@printindent@tempcnta
\def\ALG@printindent{%
    \ifnum \theALG@nested>0
    \ifx\ALG@text\ALG@x@notext
    \else 
    \unskip
    \ALG@printindent@tempcnta=1
    \loop
    \algrule[\csname ALG@ind@\the\ALG@printindent@tempcnta\endcsname]%
    \advance \ALG@printindent@tempcnta 1
    \ifnum \ALG@printindent@tempcnta<\numexpr\theALG@nested+1\relax
    \repeat
    \fi
    \fi
}
\patchcmd{\ALG@doentity}{\noindent\hskip\ALG@tlm}{\ALG@printindent}{}{\errmessage{failed to patch}}
\patchcmd{\ALG@doentity}{\item[]\nointerlineskip}{}{}{} 
\makeatother

\theoremstyle{plain}
\newtheorem{theorem}{Theorem}[section]

\newtheorem{proposition}[theorem]{Proposition}

\newtheorem{conjecture}{Conjecture}[section]
\newtheorem{corollary}[theorem]{Corollary}
\newtheorem{lemma}[theorem]{Lemma}

\theoremstyle{definition}
\newtheorem{definition}{Definition}[section]
\newtheorem{example}{Example}[section]

\theoremstyle{remark}
\newtheorem{remark}{Remark}

\DeclareMathOperator*{\E}{\mathbb{E}}

\newcommand{\mc}{\mathcal}
\newcommand{\alg}{A}
\newcommand{\Aopt}{\alg_{OPT}}
\newcommand{\Singletons}{Singletons}

\newcommand{\Nsing}{K}
\newcommand{\nsing}{k}
\newcommand{\Xone}{X_S}

\newcommand{\npops}{N}
\newcommand{\SCnoise}{\delta}
\newcommand{\qHC}{q_{HC}}
\newcommand{\inputspace}{\mc{Z}}

\newcommand{\learn}[1]{\ensuremath{\mathrm{Learn}(#1)}}
\newcommand{\learnlong}[1]{\ensuremath{\mathrm{Learn}(n,\npops,#1,\pi)}}
\newcommand{\sing}[2]{\ensuremath{\mathrm{Singletons}(#1,#2)}}
\newcommand{\bit}[1]{{{\{0,1\}}^{#1}}}
\newcommand{\metadist}{\ensuremath{q}}
\newcommand{\eps}{\epsilon}
\newcommand{\defeq}{\stackrel{{\mbox{\tiny def}}}{=}}

\ifnum\notes=1
\newcommand{\mynote}[3]{\marginpar{\tiny \sf \color{#1} {#2}: {#3}}}
\else 
\newcommand{\mynote}[3]{}
\fi

\newcommand{\DropForCOLT}[1]{}
\newcommand{\terminalbox}{\hfill $\blacksquare$}

    \title{When is Memorization of Irrelevant Training Data \\
    Necessary for High-Accuracy  Learning?
     }

    \author{Gavin Brown\thanks{Computer Science Department, Boston University. \texttt{\{grbrown,mbun,ads22\}@bu.edu}. GB and AS are supported in part by NSF award CCF-1763786 as well as a Sloan Foundation research award. MB is supported by NSF award CCF-1947889.}
    \and 
    Mark Bun\footnotemark[1]
    \and
    Vitaly Feldman\thanks{Apple.}
    \and
    Adam Smith\footnotemark[1]
    \and 
    Kunal Talwar\footnotemark[2]
    }
    \date{\today}

\begin{document}

\maketitle

\begin{abstract}
Modern machine learning models are complex and frequently encode surprising amounts of information about individual inputs. 
In extreme cases, complex models appear to memorize entire input examples, including seemingly irrelevant information (social security numbers from text, for example). In this paper, we aim to understand whether this sort of memorization is necessary for accurate learning. We describe natural prediction problems in which every sufficiently accurate training algorithm must encode, in the prediction model,  essentially all the information about a large subset of its training examples. This remains true even when the examples are high-dimensional and have entropy much higher than the sample size,
and even when most of that information is ultimately irrelevant to the task at hand. Further, our results 
do not depend on the training algorithm or the class of models used for learning. 
    
Our problems are simple and fairly natural variants of the next-symbol prediction and the cluster labeling tasks. These tasks can be seen as abstractions of text- and image-related prediction problems. To establish our results, we reduce from a family of one-way communication problems for which we prove new information complexity lower bounds. 
Additionally, we present synthetic-data experiments demonstrating successful attacks on logistic regression and neural network classifiers.
\end{abstract}

\newpage 
{\footnotesize
  \tableofcontents
}

\newpage 
\pagestyle{plain}
    \pagenumbering{arabic} 

\section{Introduction}
\label{sec:introduction}

Algorithms for supervised machine learning take in training data, attempt to extract the relevant information, and produce a prediction algorithm, also called a \emph{model} or \emph{hypothesis}. The model is used to predict a particular feature on future examples, ideally drawn from the same distribution as the training data. 
Such algorithms operate on a huge range of prediction tasks, from image classification to language translation, often involving highly sensitive data.
To succeed,  models must of course contain information about the data they were trained on. In fact, many well-known machine learning algorithms create models that explicitly encode their training data: the ``model'' for the $k$-Nearest Neighbor classification algorithm is a description of the dataset, and Support Vector Machines include points from the dataset as the ``support vectors.''
Clearly, these models can be said to memorize at least part of their training data.

Sometimes, however, memorization is an implicit, unintended side effect. In a striking recent work, Carlini et al.~\cite{carlini2020extracting} demonstrate that modern models for next-word prediction memorize large chunks of text from the training data verbatim, including personally identifiable and sensitive information such as phone numbers and addresses. Memorization of training data points by deep neural networks has also been observed in synthetic problems \citep{radhakrishnan2019overparameterized,zhang2019identity}.
The causes of this behavior are of interest to the foundations of both machine learning and privacy. For example,  a model accidentally  memorizing Social Security numbers from a text data set presents a glaring opportunity for identity theft. 

In this paper, we aim to understand when this sort of memorization is unavoidable. We give natural prediction problems in which \textit{every reasonably accurate training algorithm must encode, in the prediction model, nearly all the information about a large subset of its training examples}. Importantly, this holds even when most of that information is ultimately irrelevant to the task at hand. We show this for two types of tasks: a  next-symbol prediction task (intended to abstract language modeling tasks
) and a multiclass classification problem in which each class distribution is a simple product distribution in $\bit{d}$ (intended to abstract a range of tasks like image labeling).  
Our results hold for any algorithm, regardless of its structure. 
We prove our statements by deriving new lower bounds on the information complexity of learning, building on the formalism of \citet*{bassily2018learners}.

We note that the word ``memorization'' is commonly used in the literature to refer to the phenomenon of \emph{label memorization}, in which a learning algorithm fits arbitrarily chosen (or noisy) labels of training data points. Such memorization is a well-documented property of modern deep learning and is related to interpolation (or perfect fitting of all the training labels) \citep{zhang2016understanding,arpit2017closer,ma2018power,yun2019small}.
Feldman~\cite{feldman2020does} recently showed that, for some problems, label memorization is \textit{necessary} for achieving near-optimal accuracy on test data.
Further, Feldman and Zhang~\cite{feldman2020neural} empirically demonstrate the importance of label memorization for deep learning algorithms on standard image classification datasets. 
In contrast, we study settings in which most of the information about entire high-dimensional (and high-entropy) training examples must be encoded by near-optimal learning algorithms.

\paragraph{Problem setting}
We define a \emph{problem instance} $p$ as a distribution over labeled examples: $p\in \Delta(\mc{X})$, where $\mc{X} = \inputspace\times \mc {Y}$ is a space of examples (in $\inputspace$) paired with labels (in $\mc{Y}$).
A dataset $X\in \mc{X}^n$ is generated by sampling i.i.d.\ from such a distribution.
We use $d$ to denote the dimension of the data, so $X$ can be described in $\Theta(nd)$ bits.
In contrast to the well-known PAC model of learning, we do not explicitly consider a concept class of functions. Rather, the instance $p$ is itself drawn from a metadistribution $\metadist$, dubbed the \textit{learning task}. The learning task $\metadist$ is assumed to be known to the learner, but the specific problem instance is a priori unknown. 
We write $P$ to denote a random instance (so $P$ is a random variable, distributed according to $\metadist$) and $p$ to denote a particular realization. See Figure~\ref{fig:setting}.

\newcommand{\err}[3]{{\text{\sf err}}_{{#1},{#2}}({#3})}

The learning algorithm $\alg$ receives a sample $X\sim P^{\otimes n}$ and produces a model $M=A(X)$ that can be interpreted as a (possibly randomized) map $M:\inputspace\to \mc{Y}$. The model errs on a test data point $(z,y)$ if $M(z)\neq y$ (for simplicity, we only consider misclassification error).
The learner $\alg$'s overall error on task $q$ with sample size $n$, denoted $\err{q}{n}{A}$, is its expected error over $P$ drawn from $q$, $X$ drawn from $P^{\otimes n}$, and test point $(Z,Y)$ drawn from $P$. 
That is,
\begin{equation}
    \err{q}{n}{A} \defeq 
    \Pr_{\substack{P\sim q, \\ X\sim P^{\otimes n},  (Z,Y)\sim P, 
    \\ \text{coins of }A,M}}
    (M(Z) \neq  Y \ \text{where} \ M=A(X))
\end{equation}
For probability calculations, we often use the shorthand ``$\alg$ errs'' to denote a misclassification by $A(X)$ (the event above), so that $\Pr(\alg \text{ errs}) = \err{q}{n}{A} $.

\begin{figure*}
    \includegraphics[width=0.99\textwidth]{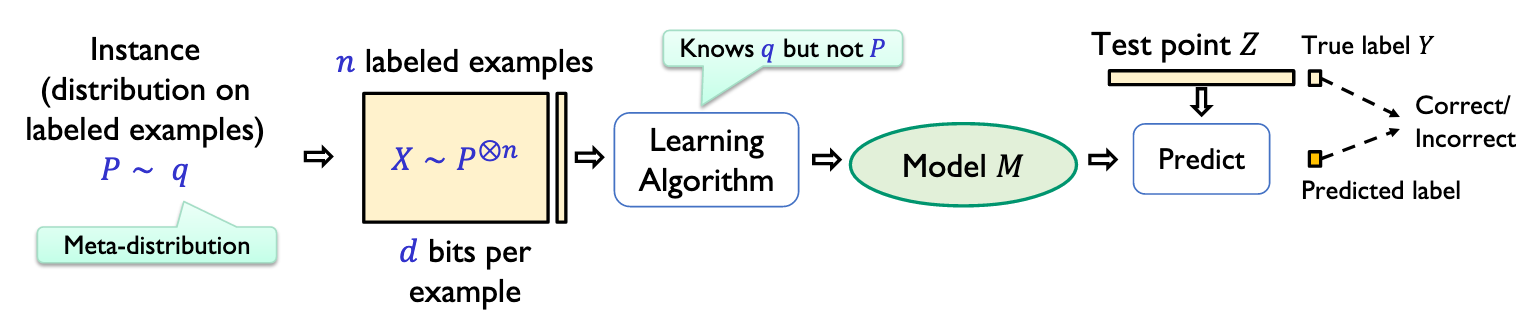}  
    \caption{Problem setting. We aim to understand the information about the data $X$ that is encoded in the model description $M$. }
    \label{fig:setting}
\end{figure*}

\begin{example} \label{ex:clustering_intro} 
    Consider the task of labeling the components in a mixture of $N$ product distributions on the hypercube. An interesting special case is a uniform  mixture of \textit{uniform distributions over subcubes}.
    Here each component $j\in [N]$ of the mixture is specified by a sparse set of fixed indices $\mc{I}_j\subseteq [d]$ with values $\{b_j(i)\}_{i\in\mc{I}_j}$. 
    Each labeled example is generated by picking at random a component $j\in [N]$ (which also serves as the label), for each $i\in \mc{I}_j$ setting $z(i)=b_j(i)$, and picking the other entries uniformly at random to obtain a feature vector $z\in \bit{d}$.  
    The labeled example is then $(z,j)$. (See Figure~\ref{fig:generating_data_GHP}.)
    A natural meta-distribution $q$ generates each set $\mc{I}_j$ by adding indices $i$ to $\mc{I}_j$ independently with some probability $\rho$, and fixes the values $b_j(i)$ at those indices uniformly at random.
    
    Given a set of $n$ labeled examples and a test point $z'$ (drawn from the same distribution, but missing its label), the learner's job is to infer the label of the mixture component which generated $z'$.
    \terminalbox
\end{example}

Given $q$, a particular meta-distribution, and $n$, the number of samples in the data set,
there exists a learner $\Aopt$ (called \textit{Bayes-optimal}) that minimizes the overall error on the task $q$. 
For any given task, this minimal error will be our reference; for $\eps\geq 0$, we call a learner $\eps$-\textit{suboptimal} (for $q$ and $n$) if its error is within $\eps$ of that of $\Aopt$ on samples of size $n$, that is, $\err{q}{n}{A} \leq \err{q}{n}{\Aopt} +\eps$.  
We have $\mc{I}_j\subseteq \hat{\mc{I}}_j$, but with few samples from cluster $j$,  $\hat{\mc{I}}_j$ will contain many irrelevant indices. 
The optimal learner will balance the Hamming distance on $\hat{\mc{I}}_j$ against the probability of achieving a set of fixed indices of that size, accounting for the fact that we expect half of the non-fixed indices to match.

\subsection{Our Contributions}
\label{sec:contribs}

We present  natural prediction problems $q$ where any algorithm with near-optimal accuracy on $q$ must memorize $\Omega(nd)$ bits about the sample. Furthermore, this memorized information is really about the specific sample $X$ and not about the data distribution $P$. 
\begin{theorem}[Informal; see Corollaries~\ref{cor:NSP_corollary_unif} and
\ref{cor:HC_corollary_unif}] 
    For all $n$ and $d$,
    there exist natural tasks $q$ for which any algorithm $\alg$ that satisfies, for small constant $\epsilon$,
    \[
    \err{q}{n}{\alg} \leq \err{q}{n}{\Aopt} + \eps 
    \]
    also satisfies
\[        I(X;M\mid P) = \Omega(nd) \, ,
    \]
    where   $P\sim q$ is the distribution on labeled examples, $X\sim P^{\otimes n}$ is a sample of size $n$ from $P$, $M=\alg(X)$ is the model, and examples lie in $\bit{d}$ (so $H(X) \leq nd$). The asymptotic expression holds for any sequence of $n,d$ pairs; the constant depends only on $\eps$. 
\end{theorem}

To interpret this result, recall that conditional mutual information is defined via two conditional entropy terms:
\(    I(X;M\mid P) = H(X\mid P) - H(X\mid M, P).
\)
Consider an informed observer who knows the full data distribution $P$ (but not $X$). The term $I(X;M |P)$ captures how the observer's uncertainty about $X$ is reduced  after observing the model $M$. 
Since $P$ is a full description of the problem instance, the term $H(X\mid P)$ captures the uncertainty about what is ``unique to the data set,'' such as noise or irrelevant features.
So $I(X;M\mid P) = \Omega(nd)$ means that \textit{not only must the learning algorithm encode a constant fraction of the information it receives, but also that a constant fraction of what it encodes is irrelevant to the task}. For one of the problems we consider, we even get that $I(X;M\mid P) = (1-o(1))H(X'\mid P)$, where $X'$ is a subset of $X$ of expected size $\Omega(n)$ and entropy $H(X'\mid P)=\Omega(nd)$ (see Theorem~\ref{thm:full-mem-informal}). That is, a subset of examples is encoded nearly completely in the model.

The meta-distribution $q$ captures the learner's initial uncertainty about the problem, and is essential to the result: if the exact distribution $P$ were known to the learner $\alg$, it could simply ignore $X$ and write down an optimal classifier for $P$ as its model $M$. In that case, we would have $I(X;M\mid P)=0$. That said, since conditional information is an expectation over realizations of $P$, our result also means that for every learner, there is a particular worst-case $p$ (in the support of $q$) such that $I(M;X)$ is large.
\ifSTOC
\else
We discuss this point further in Appendix~\ref{sec:minimax}. 
\fi
Such worst-case bounds were considered in a series of related papers \citep{bassily2018learners,nachum2018direct,nachum2019average}, with which we compare below.

Our results  lower bound mutual information. The statements do not directly shed light on whether a computationally efficient attacker, given access to the classifier, could recover some or all of the training data. 
Our proofs do suggest limited forms of recovery for some adversaries and Section~\ref{sec:experiments} demonstrates successful black-box attacks against simple neural networks,
but we leave the full investigation of efficient recovery, and attacks against specific learning algorithms, as areas for future research.

We study two classes of learning tasks. The first is a next-symbol prediction problem (intended to abstract language modeling tasks). The second is the cluster labeling 
\ifSTOC
\else
    problem (a generalization of Example~\ref{ex:clustering_intro} that allows more natural mixture weights).
\fi
The exact problems are defined in Section~\ref{sec:techniques}. 

In all the tasks we consider, data are drawn from a mixture of subpopulations. We consider settings where there are likely to be $\Omega(n)$ components of the mixture distribution from which the data set contains exactly one example. Leveraging new communication complexity bounds, we show that $\Omega(d)$ bits about most of these ``singleton'' examples must be encoded in $M$ for the algorithm to perform well on average.

Returning to the cluster labeling problem in Example~\ref{ex:clustering_intro}, recall that the learner receives an $nd$-bit data set, which has entropy $\Theta(nd)$, even conditioned on $P$  (when $\rho$, the probability of fixing an index, is bounded away from 1). This ``remaining uncertainty'' $H(X\mid P)$ is, ignoring lower-order terms, exactly the uncertainty about the values of the irrelevant features. Showing $I(X;M\mid P) = \Omega(nd)$, then, establishes not only that the model must contain a large amount of information about $X$, but also that it must encode a large amount of information about the unfixed features, information completely irrelevant to the classification task at hand.

On a technical level, our results are related to those of 
\citet*{bassily2018learners}, \cite{nachum2018direct} and \citet{nachum2019average}, who study lower bounds on the mutual information $I(X; M)$ achievable by a PAC learner for a given class of Boolean functions $\mc{H}$.
Specifically, for the class $\mc{H}_{thresh}$ of threshold functions on $[2^d]$, they give a learning task%
\footnote{The results of \citet{bassily2018learners,nachum2018direct,nachum2019average} are formulated in terms of worst-case information leakage over a class of problems. They imply the existence of a single hard meta-distribution $q$ by a minimax argument.} 
for which  every \textit{proper and consistent} learning algorithm (i.e. one that is limited to outputting a function in $\mc{H}_{thresh}$ that labels the training data perfectly) satisfies
$I(X; M\mid P) =\Omega\left(\log d\right)$ \citep{bassily2018learners,nachum2018direct}. Furthermore, \citet{nachum2018direct} extend this result to provide
a hypothesis class $\mc{H}$ with VC dimension $n$ over the input space  $[n]\times \bit{d}$  such that learners receiving $\Omega(n)$ samples must leak at least $ I(X; M\mid P) = \Omega\left(n \cdot \log (d - \log n) \right)$ bits about the input via their model. 
The direct sum construction in \cite{nachum2018direct} is similar to our construction: they build a learning problem out of a product of simpler problems 
and relate the difficulty of the overall problem to that of the components. 

Even more closely related is concurrent work of \citet[Theorem 2, setting $m=2$]{livni2020limitation}, which gives settings in which  the PAC-Bayes framework cannot yield good generalization bounds. Their result implies  that sufficiently accurate algorithms for learning thresholds over $[2^d]$ must leak $\Omega(\log \log d)$ bits of information.
(This can be extended to a lower bound of $\Omega(n\log\log d)$ for learning products of thresholds from samples of size $n$.) 
It is unclear if those techniques can yield bounds that scale linearly with $d$.

As we show in Appendix~\ref{sec:additional_models}, 
our results on next-symbol prediction can be cast in terms of learning threshold functions. As such, our results provide an  alternative  to those of \cite{bassily2014private}, \cite{nachum2018direct}, and \cite{livni2020limitation}. First, they are quantitatively stronger: we give a lower bound of 
$(1-o(1)) nd$
rather than $\Omega (n \log d)$ or $\Omega(n \log\log d)$. Second, our bounds and those of \cite{livni2020limitation} apply to all sufficiently accurate learners, whereas those of \cite{bassily2014private} and \cite{nachum2018direct} require the learner to be  proper and consistent (an incomparable assumption, in the regimes of interest).

\paragraph{Implications} 
While the problems we describe are intentionally simplified to allow for clean theoretical analysis, 
they rely on properties found in natural learning problems such as clustering of examples, noise, and a fine-grained subpopulation structure~\citep{zhu2014capturing}.
Our results thus suggest that memorization of irrelevant information, observed in practice, is a fundamental property of learning and not an artifact of particular deep learning techniques. 

Our proofs rely on an assumption of independence between subpopulations.
While this is a natural assumption for mixture models broadly, it is a significant simplification for a model of natural language or images.
We believe that one could prove weaker but still meaningful statements about memorization under relaxed versions of the independence assumption.
The crucial ingredients are that (i) samples contain useful information about their subpopulation alongside irrelevant information and (ii) the learning algorithm is unable to discern which is which.
Independent subpopulations make for easier proofs and cleaner statements,
but do not seem to be a requirement for memorization.

Our results have implications for learning algorithms that (implicitly) limit memorization. One class of such algorithms aims to compress models (for example to reduce memory usage), since description length upper bounds the mutual information. Differentially private algorithms~\citep{DworkMNS06} form another such class. It is known that differential privacy implies a bound on the mutual information  $I(X;M\mid P)$ \citep{McGregorMPRT10,dwork2015generalization,rogers2016max,bun2016concentrated}. Our results imply that such algorithms might not be able to achieve the same accuracy as unrestricted algorithms. In itself, that is nothing new: there is a long line of work on differentially private learning  \citep[for example][]{BlumDMN05, kasiviswanathan2011can}, including a number of striking separations from nonprivate algorithms \citep{bun2018fingerprinting,bassily2014private,AlonLMM19}. There are also well-established attacks on statistical and learning algorithms for high-dimensional problems (starting with \citep{dinur2003revealing}; see \citep{dwork2017exposed} for a survey of the theory, and a recent line of work on membership inference attacks \citep{shokri2017membership} for empirical results).
However, our results show a novel aspect of the limits of private learning: in the settings we consider, \textit{successful learners must memorize exactly those parts of the data that are most likely to be sensitive}—unique samples from small subpopulations, including their peculiar details (modeled here as noisy or irrelevant features).

\paragraph{Variations on the main result}
Different learning tasks exhibit variations and refinements of this central result.
The mutual information lower bound implies that the model itself must be large, occupying at least $\Omega(nd)$ bits. 
But for some tasks we present, there exist $\eps$-suboptimal models needing only $O(n \log (n/\eps) \log d)$ bits to write down (in the parameter regime we consider, where the problem scales with $n$). 
That is, with $n$ samples the learning algorithm must output a model exponentially larger than what is required with sufficient data (or exact knowledge of $P$).
In particular, for a given target accuracy level, there is a gradual drop in the size of the model, and the information specific to $X$, that is necessary (starting at $\Theta(n_0 d)$ where $n_0$ is the minimal sample size needed for that accuracy, and tending to $O(n_0 \log n_0 \log d)$ as the sample size $n$ grows). 
For task-specific discussion, see Section~\ref{sec:GHP} Remark \ref{rem:HC_small_representation}, and Appendix \ref{sec:2_NSP}.

Another variation of our results gives a qualitatively stronger lower bound.  For some tasks, we are able to demonstrate that \textit{entire} samples must be memorized, in the following sense: 
\begin{theorem}[Informal; see Corollary~\ref{cor:NSP_corollary_unif} and Conjecture~\ref{conj:HC_singletons_bound}] \label{thm:full-mem-informal}
There exist natural tasks $q$ for which every data set $X$ has a subset of records $X_S$ such that 
\begin{compactitem}
    \item $\E[|X_S|]=\Omega(n)$, $H(X_S\mid P) = \Omega(nd)$, and
    \item any algorithm $\alg$ that satisfies \(
    \err{q}{n}{\alg} \leq \err{q}{n}{\Aopt} + \eps 
    \)
    also satisfies
\[        I(X_S;M\mid P) = (1-o(1)) H(X_S\mid P)  \, .
    \]
\end{compactitem}
\end{theorem}
This statement implies $I(X;M\mid P) = \Omega(nd)$,
 but is a qualitatively different statement: as the learning algorithm's accuracy approaches optimal, it must reveal \emph{everything} about these examples in $\Xone$, at least information-theoretically.
For these tasks, there is no costless compression the learning algorithm can apply: any reduction will increase the achievable error.
    We conjecture that this ``whole-sample memorization'' type of lower bound applies to all the tasks we study, and in fact give a simple conjecture on one-way information complexity that would imply such strong memorization (see below and Section~\ref{sec:GH_conjecture}). 

In addition to the memorization of noise or irrelevant features, in some settings we show how near-optimal models may be forced to memorize examples that are themselves entirely ``useless,'' i.e. could be ignored with only a negligible loss in accuracy. That is, not only must irrelevant details of useful examples be memorized, but one must also memorize examples that come from very low-probability parts of the distribution. 
Unlike our main results, which hold for uniform mixtures over the subpopulations, this behavior relies on a particular type of mixture structure and the long-tailed distribution setup of \cite{feldman2020does}. 
We explain the concept and the statement more carefully in 
    Section~\ref{sec:long_tail}.

Finally, complementing our information-theoretic lower bounds, experiments in Section~\ref{sec:experiments} demonstrate efficient data reconstruction from concrete learning algorithms.
Generating synthetic data sets from the cluster labeling problem described in Example~\ref{ex:clustering_intro}, we train multiclass logistic regression and neural network classifiers to high accuracy.
An adversary with black-box access to the trained classifer then tries to recover singleton examples, using simple attacks such as approximately maximizing the target class probability.
These attacks are extremely successful in our experiments, on average recovering over 97\% of the singletons' bits.

\subsection{Techniques: Subpopulations, Singletons, and Information Complexity}
\label{sec:techniques}

The learning tasks $q$ we consider share a basic structure: each distribution $P$ consists of a mixture, with coefficients $D\in \Delta([N])$, over subpopulations $j\in [N]$, each with a different ``component distribution'' $C_j$ over labeled examples. The mixture coefficients $D$ may be deterministic (e.g.\ uniform) or random; 
    for now, the reader may keep in mind the uniform mixture setting,
with $N=n$ (so the number of subpopulations is the same as the sample size). The $C_j$'s are themselves sampled i.i.d.\ from a meta-distribution $q_c$. 

As in \citep{feldman2020does}, we look at how the learning algorithm behaves on the subset of examples that are \textit{singletons}, that is, sole representatives (in $X$) of their subpopulation. For any data set $X$, let $X_S\subseteq X$ denote the subset of singletons. We consider mixture weights $D$ where $X_S$ has size $\Omega(n)$ with high probability. We show that for our tasks, a successful learner must roughly satisfy $I(X_S;M\mid P)=\Omega(d|X_S|)$, where $d$ is the dimension of the data.
Our results rely on the learning algorithm doing almost as well as possible with the size-$n$ sample they are given. 
That requires us to adapt the distribution to $n$.
For any fixed distribution we consider, if the sample is large enough, our proofs will yield weaker guarantees.
For instance, if instead of $n=N$ samples from the uniform mixture over subpopulations, we draw $n=2N$, then we will get fewer singletons, although we still expect $|X_S|=\Omega(n)$.
If we increase the sample size to $n=\Omega(N \log N)$, with high probability the data set will contain no singletons.

\paragraph{One-Way Information Complexity of Singletons}
We show that a good learner implies a good strategy for a related one-way communication game, 
    dubbed \Singletons$(k,q_c)$. 
In this game, nature generates $k$ distributions $C_1,\dots,C_k$ i.i.d.\ from the meta-distribution on clusters $q_c$, along with a uniformly random index $j^* \in [k]$. One player, Alice, receives a list $(x_1,\dots,x_k)$ of labeled examples, where $x_j\sim C_j$. A second player Bob, receives only the feature vector $z$ from a fresh draw $(z,y)\sim C_{j^*}$. Alice sends a single message $M$ to Bob, who predicts a label $\hat y$.  Alice and Bob win if $\hat y = y$. 

\begin{example}[Nearest neighbor, Figure~\ref{fig:generating_data_GHP}]\label{ex:nearest_neighbor_revisted}
For the hypercube task corresponding to Example~\ref{ex:clustering_intro}, let $q_{HC}$ be the distribution from which the $\{C_j\}$ are sampled.
    In $\sing{k}{q_{HC}}$, 
there are $k$ sets of fixed indices $\{(\mc{I}_j,b_j)\}_{j\in[k]}$.
Alice gets a list $X'=(x_1,\dots,x_k)\in (\bit{d})^k$, where for every example $j$, we have: $\forall i\in\mc{I}_j, x_j(i)=b_j(i)$ and $\forall i\notin\mc{I}_j$, $x_j(i)=\mathrm{Bernoulli}(1/2)$. The label, $j$, is implicit in the ordered list. 
Bob receives $z$ for a random index $j^*$ and must predict $j^*$.

Equivalently, we may view the game as a version of the nearest neighbor problem, treating Alice's input list $(x_1,\dots,x_k)$ as uniformly random in $(\bit{d})^k$ and Bob's input as a corrupted version of the one of the $x_j$'s.
If each $\mc{I}_j$ is built by adding every index independently with probability $\rho$,
one can quickly check that generating $z$  from the same distribution as $x_j$ is equivalent to setting $z = BSC_{\frac{1-\rho}{2}}(x_{j^*})$,  where $BSC_{\frac{1-\rho}{2}}$ is the binary symmetric channel that flips each bit of $x_{j^*}$ independently with probability $\frac{1-\rho}{2}$.
If Bob were to see Alice's entire input, his best strategy would be to guess index of the point in $(x_1,...,x_k)$ that is nearest to $z$. One can show that he succeeds with high probability as long as $\rho \approx \sqrt{\frac{\ln k}{d}}$.  

This straightforward strategy requires Alice to send $nd$ bits. We ask: can Bob still succeed with high probability when Alice sends $o(nd)$ bits?
\terminalbox
\end{example}

\begin{figure}
    \centering
    \begin{subfigure}[b]{0.4\textwidth}
        \includegraphics[width=\textwidth, page=3, clip=true, trim= 1.5in 4in 7.9in 1in]{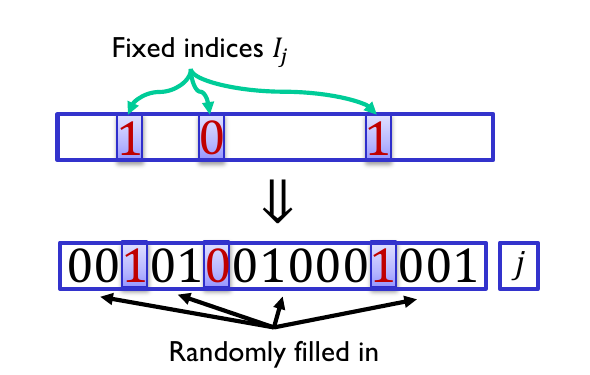}
        \caption{Hypercube Cluster Labeling}
        \label{fig:generating_data_GHP}
    \end{subfigure}
    ~
    \begin{subfigure}[b]{0.4\textwidth}
        \includegraphics[width=\textwidth, page=3, clip=true, trim= 7in 3.9in 2.5in 1in]{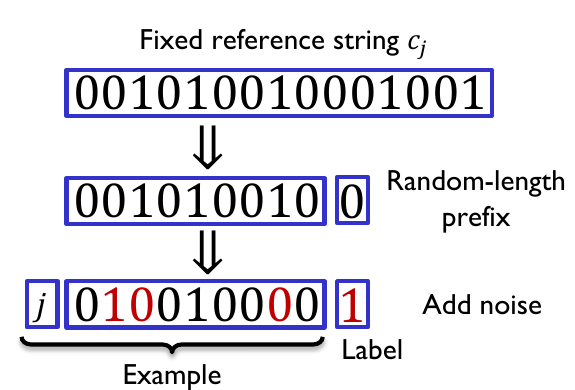}
        \caption{Next-Symbol Prediction}
        \label{fig:generating_data_NSP}
    \end{subfigure}
    \caption{(a) In hypercube cluster labeling, each subpopulation $j$ is associated with a sparse set of fixed indices. The example is generated by filling in the other indices randomly. The label is $j$.
    (b) In next-symbol prediction, each subpopulation $j$ is associated with a reference string. Examples contain $j$ paired with a noisy prefix of random length. The label is the next bit, which may also be corrupted.}
\end{figure}

One novel technical result bounds the information complexity of this nearest neighbor problem.

\ifSTOC
\else
    \begin{lemma}[informal; see Lemma~\ref{cor:HC_singletons_bound}]\label{lemma:informal_HC_sing}
        Set $\rho = \sqrt{\frac{2\ln ak - \ln \ln ak}{d}}$ for $a>1$, a constant.
        For all $k$ sufficiently large, $d\ge k^{0.1}$, and $\epsilon_k$ sufficiently small, the information complexity of (one-way) $\epsilon_k$-suboptimal protocols for Singletons$(k,q_{HC})$ (Example~\ref{ex:nearest_neighbor_revisted}) is $I(X'; M)\ge \Omega(kd)$.
    \end{lemma}
\fi

We prove this using the strong data processing inequality for binary symmetric channels, directly analogous to its recent use in bounding the one-way information complexity of the Gap-Hamming problem \citep{hadar2019communication}. 

The proof of this result is subtle, and does not proceed by separately bounding the information complexity of solving each of the $k$ subproblems implicit in
\ifSTOC
\else
    \Singletons$(k,q_{HC})$. 
\fi
The parameter $\rho$ is large enough that one can reliably detect proximity to any \textit{one} of Alice's inputs with a message of size $\frac{d}{\log k} = o(d)$. 
Our proof crucially uses the fact that Bob must select from among $k$ possibilities. It shows that his optimal strategy is to detect proximity to each of Alice's inputs with failure probability $\approx 1/k$,  controlling his total failure probability by a union bound.

\ifSTOC
\else
The proof of Lemma~\ref{lemma:informal_HC_sing} (later, Lemma~\ref{cor:HC_singletons_bound}) cannot yield a lower bound stronger than $\frac{1}{2\ln 2}\cdot kd$.
We conjecture in Section~\ref{sec:GH_conjecture} that the constant factor of $\frac{1}{2\ln 2}$ can be improved to 1, matching exactly the upper bound from the naive algorithm.
In a related ``one-shot'' case of the Gap-Hamming problem, where the data points are independent and uniform, we are able to prove $I(X;M)\ge (1-o(1))H(X)$.
This result, Theorem~\ref{thm:GHP_oneshot_main} in Appendix~\ref{sec:one_way_GHP}, suggests that no more sophisticated algorithm exists.
The information complexity of Gap-Hamming is known to be $\Omega(n)$ for both one- and two-way communication, but we are not aware of a result showing a leading constant of 1.
\fi

\paragraph{Next-bit Prediction and One-Shot Learning}
Inspired by the empirical results of \cite{carlini2018secret,carlini2020extracting}, we demonstrate a sequence prediction task which requires memorization.
Each subpopulation $j$ is associated with a fixed ``reference string,'' and samples from the subpopulation are noisy prefixes of this string.

\begin{example}[Next-Symbol Prediction, Figure~\ref{fig:generating_data_NSP}]\label{ex:nsp}
    In the next-symbol prediction task the component distribution $q_{NSP}$ draws a reference string $c_j\in\{0,1\}^d$ uniformly at random.
    Samples from $j$ are generated by randomly picking a length $\ell\in\{0,\ldots, d-1\}$, then generating $z\sim BSC_{\SCnoise/2}(c_j(1:\ell))$ for some noise parameter $\SCnoise\in[0,1)$.
    We pair $z$ with a subpopulation identifier, so the example is $(j,z)$.
    The label is a noisy version of the next bit: $y\sim BSC_{\SCnoise/2}(c_j(\ell+1))$.
\terminalbox
\end{example}

Unlike cluster problems, where the label \textit{is} the subpopulation, each subpopulation can be treated independently by the learning algorithm.
The core of our lower bound for this task, then, is to prove a ``one-shot'' lower bound on the setting where both Alice and Bob each receive a single example from the same subpopulation.
\ifSTOC
    \begin{lemma}[Informal; see Lemma \ref{lemma:SC_one_shot}]
\else
    \begin{lemma}[informal; see Lemma \ref{lemma:SC_one_shot_second}]
\fi
    For sufficiently small $\eps$, any algorithm that is $\eps$-suboptimal on 
    \ifSTOC
\else
        (noiseless) Singletons$(1,q_{NSP})$ 
    \fi
    satisfies
    \begin{align}
        I(X; M)\ge \frac{d+1}{2} \left(1 - h\left(2 \eps\right)\right).
    \end{align} 
\end{lemma}
Note that $\frac{d+1}{2}$ is the average length of Alice's input and that the $\log d$ term arises from uncertainty about that length, which Alice need not convey.
The proof proceeds by establishing that Bob's correctness is tied to his ability to output Alice's relevant bit.
For any fixed length of Alice's input, the problem is similar to a communication complexity problem called Augmented Indexing.
We adapt the approach of a well-known elementary proof \citep{bar2004information,feldman2014sample}.

\ifSTOC
\else
    \subsection{Related Concepts: Representation Complexity, 
               Time-Space Tradeoffs, and Information Bottlenecks} 
\fi

Our results are closely related to a number of other lines of work in machine learning. 
First, as discussed in the introduction, one can view our results as a significant strengthening of recent results on label memorization \citep{feldman2020does} and information-theoretic lower bounds for learning thresholds \citep{bassily2018learners,nachum2018direct,livni2020limitation}.




\newcommand{\RC}{\ensuremath{\text{\sf PRep}_\eps (C)}}
\paragraph{Representation Complexity}
Another closely related concept is \textit{probabilistic representation complexity}~\citep{beimel2019characterizing, feldman2014sample}. For given error parameter $\epsilon$, the representation complexity $\RC$  of a class $C$ of concepts (functions from $\mc{Z}$ to $\mc Y$) is roughly the smallest number of bits needed to represent a hypothesis that  approximates (up to expected error $\eps$) a concept $c\in C$ on an example distribution $P_z\in \Delta(\mc{Z})$, in the worst case over pairs $(c,P_z)$.\footnote{See \citep{feldman2014sample} for an exact definition.}  This complexity measure characterizes the sample complexity of ``pure'' differentially private learners for $C$~\citep{beimel2019characterizing}.

Interpreted in our setting, representation complexity aims to understand the \textit{length} of the message $M$, when the task $q$ is a distribution over pairs $(c,P_z)$ (that is, where the data distribution $P$ consists of examples drawn from $P_z$ and labeled with $c$). By a minimax argument, one can show that $\RC $ lower bounds not only $M$'s length, but also the information it contains about $P$: one can find $q$ such that $I(P;M)$ is at least $\RC$. This does imply that $I(X;M)$ must be large, but it says nothing about the information in $M$ that is specific to a particular sample $X$: in fact, the bound is saturated by learners that get enough data to construct a hypothesis that is just a function of $P$, so that $I(X;M\mid P)$ is  small. 

The bounds we prove here are qualitatively stronger. We give settings where the analogue of representation complexity is small (namely, a learner that knows $P$ can construct a model of size about $n \log (n/\eps)\log d$), but where a learner which only gets a training sample must write down a very large model ($\Omega(nd)$ bits) to do well.

\paragraph{Time-Space Tradeoffs for Learning} 
A recent line of work establishes time-space tradeoffs for learning in the streaming setting: problems where any learning algorithm requires either a large memory or a large number of samples
(see \citep{garg2018extractor} for a summary of results).
The prime example is parity learning over $d$ bits, which is shown to require either $\Omega(d^2)$ bits of memory or exponentially many samples.
The straightforward algorithm for parity learning requires $O(d)$ samples, so this result shows that any feasible algorithm must store, up to constant factors, as many bits as are required to store the dataset~\citep{raz2018fast}.

Our work sets a specific number of samples under which learning is feasible and, for that number of samples, establishes an information lower bound on the \emph{output} of the algorithm.
This implies not only a communication lower bound but also one on memory usage: the algorithm must store the model immediately prior to releasing it.
Some of our tasks exhibit the property that, with additional data, an algorithm can output a substantially smaller model.
These learning tasks might exhibit a time-space tradeoff, although not one as dramatic as the requirement of exponentially many samples.
Intuitively, the underlying concept in parity learning must be learned ``all at once.''
Our problem instances can be learned ``piece-by-piece,'' as the algorithm learns sections independently of the rest of the sample.

\paragraph{Information Bottlenecks}
%
Our work fits into the broad category of information bottleneck results in information theory~\citep{tishby2000information}. An information bottleneck is a compression scheme for extracting from a random variable $V$ all the information relevant for the estimation of another random variable $W$ while discarding all irrelevant information in $V$.
In our setting, one may take $V=X$ to be the data set, and $W$ to be the true distribution $P$ (where the loss of a model is its misclassification error). 
This general form of information bottleneck was recently described in independent work \citep{alemi2020variational}.
Our results lower bound the extent to which nontrivial compression is possible, showing that the Markov chain $P-X-M$ must in particular satisfy $I(X;M)\gg I(M;P)$. 

Information bottlenecks have been put forward as a theory of how implicit feature representations evolve during training~\citep{tishby2015deep}. 
That line of work studies how the prediction process transforms  information from a test datum during prediction (i.e.\ as one moves through layers of a neural network), and is thus distinct from our study of how {learning algorithms} are able to extract information from {training data sets}.

\ifSTOC
\else
    \subsection{Organization of This Paper}
    In Section~\ref{sec:long_tail} we specify our general framework for learning tasks, detailing how mixture coefficients and subpopulations are sampled.
    We also introduce and prove our ``central reduction,'' showing how an algorithm for a learning task provides an algorithm for solving the singletons-only task.
    In Sections~\ref{sec:SC} and~\ref{sec:GHP} we formally define Next-Symbol Prediction and Hypercube Cluster Labeling
    and provide lower bounds on $I(X;M\mid P)$ for both tasks.
    Section~\ref{sec:experiments} presents simple experimental results from attacking neural networks trained on synthetic data generated from the hypercube cluster labeling task.
    
    In Appendix~\ref{sec:technical_details} we group miscellaneous technical results needed elsewhere in the paper.
    In Appendix~\ref{sec:additional_models}, we prove results for related learning tasks, including a lower bound for threshold learning.
    Appendix~\ref{sec:DP_calc} presents a simple connection to differential privacy.
    In Appendix~\ref{sec:minimax}, we argue via minimax that our average-case lower bounds imply similar worst-care guarantees.
    Finally, in Appendix~\ref{sec:one_way_GHP}, we present our $(1-o(1))\cdot d$ proof for the one-way Gap Hamming problem, which we believe provides evidence for Conjecture~\ref{conj:HC_singletons_bound}.
\fi

\section{Subpopulations, Singletons, and Long Tails}
\label{sec:long_tail}

Before analyzing the specific learning tasks, we present the key points of the framework upon which our tasks our built.
We also outline the high-level structure of our task-specific bounds.

Recall that we define a problem instance $P$, which is a random variable drawn from a metadistribution $q$, to be a distribution over labeled data.
In our work $P$ will be a mixture over $\npops$ \emph{subpopulations}, each of which may have its own distribution, label, or classification rules.
We decompose $P=(D,C)$, where $D$ is a list of $\npops$ mixture coefficients and $C$ is a list of $\npops$ distributions over labeled examples, one for each subpopulation.
To sample a data point from a problem instance $p$, we first sample a subpopulation $j\sim D$ and then sample the labeled point $(z,y)\sim C_j$.

The metadistribution $q$ is specified by two generative processes, one for generating $D$ and the other for generating $C$.
(Formally, we will take $C$ to be parameters, not distributions, but ignore the distinction for now.)
The first process, described below in Section \ref{sec:generating_mixtures}, depends only on $\npops$ and a list of frequencies $\pi$, which we refer to as a ``prior.''
The details of the second process will be task-specific, but for each task there will be a ``component distribution'' $q_c\in\Delta(\mc{X})$ from which the entries in $C$ are sampled i.i.d.

The learning task is thus completely determined by the sample size $n$, the number of subpopulations $N$, the (task-specific) component distribution $q_c$, and the prior $\pi$.
We give this standard setting a name.
\begin{definition}\label{def:learn}
    We call our standard learning task \learnlong{q_c}.
    Problem instance $P=(D,C)$ is generated from $q$.
    A data set of $n$ i.i.d.\ samples are drawn from $p$ and given to the learning algorithm.
    One test sample $(z,y)$ is drawn independently from $p$, and the model predicts a label.
    \terminalbox
\end{definition}
When the other terms are clear from context, we will shorten this to \learn{q_c}, since only the component distribution will change from task to task.

Our results rely on the analyzing how algorithms perform on subpopulations for which they receive exactly one data point.
We call these points \emph{singletons}.
To capture this behavior, we define a second type of task.
This is also a learning task, but, unlike in \learnlong{q_c}, the samples are no longer i.i.d. from a mixture.
\begin{definition}
    We denote by \sing{k}{q_c} the singletons-only task on $k$ subpopulations. 
    In this task $k$ subpopulation parameters $C_j$ are sampled i.i.d. from $q_c$, and from each $C_j$ is sampled exactly one labeled data point.
    These $k$ samples form the data set given to the learner.
    There is an index $j^*\in[k]$ sampled uniformly at random; the test sample is drawn from $C_{j^*}$.
    \terminalbox
\end{definition}
We return to \Singletons($k,q_c$), and its relationship to \learnlong{q_c}, in Section \ref{sec:central_reduction}.

\subsection{Generating Mixtures over Subpopulations}\label{sec:generating_mixtures}

We generate a mixture over $\npops$ subpopulations using the process introduced in \cite{feldman2020does}.
Although our central results will hold in the setting where the mixture is uniform (and thus chosen without randomness), 
this process sets up a \emph{qualitatively different type} of ``memorization of useless information,'' an example of which is crystallized in Example \ref{ex:bimodal} and occurs naturally in long-tailed distributions.
We encourage the reader to keep the uniform case in mind for simplicity but remember that the results apply to broad settings exhibiting varied behavior.

Starting with a list $\pi$ of nonnegative values, for each subpopulation $j\in [N]$, we sample $\delta_j \sim \mathrm{Uniform}(\pi)$.
To create distribution $D$, we normalize:
\begin{align*}
    D(j) = \frac{\delta_j}{\sum_{i\in [N]} \delta_i}.
\end{align*}

This quasi-independent sampling facilitates certain computations.
In particular, we will want to quantify the following: given that subpopulation $j$ has exactly one representative in data set $X$, what is the probability the test sample $(z,y)$ comes from the same subpopulation?
The answer can be computed as a function of $n, \npops$, and $\pi$, independently of $q_c$ and, crucially, the rest of the data set.
We call this quantity
\begin{align}
    \tau_1 \defeq \Pr[\text{$(z,y)$ comes from $j$} \mid \text{$\Xone$ contains one sample from $j$}], \label{eq:tau1_def}
\end{align}
defining $\Xone$ to be the data set restricted to those singletons.
By linearity of expectation we have
    $\Pr[\text{$(z,y)$ comes from a singleton subpopulation}] = \tau_1 \times |\Xone|$.

We will also need to refer to the expected size of $|\Xone|$.
Like $\tau_1$, this is a function only of $n, \npops$, and $\pi$.
We have, defining the quantity as a fraction of the data set,
\begin{align}
    \mu_1 \defeq  \frac{\E[|\Xone|]}{n}. \label{eq:mu1_def}
\end{align}
Our memorization results are most striking when $\tau_1=\Omega(1/n)$ and $\mu_1=\Omega(1)$.
We provide two simple examples of when this is so.

\begin{example}[Uniform]\label{ex:uniform}
    If the list of frequencies is a single entry $\pi = (1/\npops)$, then the mixture over subpopulations will be uniform.
    Set $n=\npops$, so the number of examples is equal to the number of bins.
    For every subpopulation the probability a test example comes from it is exactly $\frac{1}{n}$, so $\tau_1=\frac{1}{n}$ as well.
    The expected fraction of singletons is also constant: we have by linearity of expectation that
    \begin{align*}
        \mu_1 &= n \cdot\frac{ \Pr[\text{subpopulation 1 receives 1 sample}]}{n} 
            = \left(1 - \frac{1}{n}\right)^{n-1} \approx \frac{1}{e}.
    \end{align*}
    \terminalbox
\end{example}

Our central results apply cleanly to the uniform setting.
In this setting, however, every example is ``important,'' i.e.\ memorizing it will provide a significant gain in accuracy.
In general this is not the case.

\begin{example}[Bimodal]\label{ex:bimodal}
    We sketch the main points and include a full description in Appendix \ref{sec:long_tail_details}.
    Suppose there are $N=2^n$ subpopulations, and the prior is such that
    \begin{align}
        \delta_j \sim \mathrm{Uniform}(\pi) = \begin{cases} \frac{1}{2n} & \text{w.p. $n2^{-n}$} \\
                        \frac{1}{2\cdot 2^{n}} & \text{w.p. $1 - n2^{-n}$}. \end{cases}.
    \end{align}
    The exact probabilities will depend on the normalization constant $C=\sum_j \delta_j$ that, as a sum of independent random variables, will exhibit tight concentration about its mean.
    Since $\E[C] \approx 1$, the mixture coefficients won't change too much after normalization.
    
    Call the bins with mass around $\frac{1}{2n}$ ``heavy'' and the others ``light.''
    Observe that about half the probability mass will lie in each group.
    Almost all the balls that go into light bins will be singletons; a constant fraction of the balls that go into heavy bins will be singletons.
    Thus $\mu_1=\Omega(1)$ and, given that a subpopulation has a single representative, with constant probability it will be a heavy bin with mass $\Omega(1/n)$, so we have $\tau_1=\Omega(1/n)$.
    But the light subpopulations which received points are unlikely to receive the test sample, and could be ignored with only an exponentially small increase in expected error, if only they were identified as light.
    \terminalbox
\end{example}

The bimodal example is stylized to show an extreme version of the phenomenon, but the concept of ``useless'' singletons arises in natural distributions.
Informally, these ``long-tailed distributions'' have a significant portion of the distribution represented by many rare instances.
A central motivation for \cite{feldman2020does}, these distributions arise in practice and suggest that success on large-scale learning tasks may depend heavily on how the algorithm deals with these atypical instances.
Like with the bimodal distribution, the learning algorithm will be unable to distinguish between examples that are very rare (and can be ignored) and those that represent an $\Omega(1/n)$ probability mass, which must be dealt with to perform near optimally on the whole task.

\subsection{Central Reduction: Singletons Task to General Learning}\label{sec:central_reduction}

We previously defined \learn{q_c} and \Singletons($k,q_c$), two distinct tasks.
The former is the focus of our interest but the latter proves more amenable to analysis.
Informally, an algorithm solving \learn{q_c} to near-optimal error will have to perform well on $X_S$.
Let us quantify ``near optimal error,'' which applies to both \learnlong{q_c} and \sing{k}{q_c}.
\begin{definition}[$\eps$-suboptimality]
    An algorithm $\alg$ is \textit{$\eps$-suboptimal on task $T$}, where $T$ is associated with metadistribution $q$, if
    \begin{align*}
        \err{q}{n}{A} \le \inf_{\alg'} \err{q}{n}{A'} + \eps.
    \end{align*}
    In our shorthand of abbreviating the event ``the model $M$ output by $\alg$ makes an error'' as ``$\alg$ errs,'' this is 
    denoted
    $    
        \Pr[\alg\text{ errs on $T$}] \le \inf_{\alg'} \Pr[\alg'\text{ errs on $T$}] + \eps.
    $
    \terminalbox
\end{definition}

Let random variable $K=|\Xone|$ be the number of singletons in the data set.
If we have an algorithm performing well on $\Xone$ in \learn{q_c} when $K=k$, we can modify it to create an algorithm performing well on \sing{k}{q_c}.
This allows us to apply lower bounds proved for the singletons problem.

\begin{lemma}[Central Reduction, Task-Agnostic]\label{lemma:central_reduction}
    Suppose we have the following lower bound for every $k$: any algorithm
    $\alg^k(X')$ 
    that is $\eps_k$-suboptimal for \sing{\nsing}{q_c} satisfies 
    \begin{align*}
        I(X';\alg^k(X')) \ge f_k(\epsilon_k).
    \end{align*}
    
    Then for any algorithm $\alg(X)$ that is $\eps$-suboptimal on \learn{q_c} there exists a sequence $\{\epsilon_k\}_{k=1}^n$ such that $\E_k[k\eps_k] \le \frac{\eps + \phi_1(q_c) + \phi_2(q_c)}{\tau_1}$ and $I(X_S;M\mid K) \ge \E_k[f_k(\epsilon_k)]$.
    
    Furthermore, if $f_k(\epsilon_k) \ge k \cdot g(\epsilon_k)$ for some convex and nonincreasing $g(\cdot)$, then
    \begin{align*}
        I(\Xone;M\mid K) &\ge \mu_1 n \cdot g\left(\frac{1}{\mu_1n}\cdot \frac{\epsilon + \phi_1(q_c) + \phi_2(q_c)}{\tau_1}\right).
    \end{align*}
    Here $\phi_1(q_c) $ and $ \phi_2(q_c)$ are task-specific terms.
    Letting $E_1$ be the event that the test sample comes from a subpopulation with exactly one representative, they are defined as
    \begin{align*}
        \phi_1(q_c) &\defeq \Pr[\bar{E_1}]\left(\Pr[\text{$\Aopt$ errs on \learn{q_c}}\mid \bar{E_1}] 
             - \Pr[\text{$\alg$ errs on \learn{q_c}}\mid \bar{E_1}]\right), \\
        \phi_2(q_c) &\defeq \sum_{k=1}^{\npops} \Pr[K=k\mid E_1]\Bigl( \Pr[\Aopt \text{ errs on \learn{q_c}}\mid E_1, K=k] \\
        &\hspace{6cm} - \inf_{\alg'} \Pr[\alg' \text{ errs on \sing{k}{q_c}}]\Bigr).
    \end{align*}
\end{lemma}

\newcommand{\Acomp}{\Aopt}
\begin{proof}
    We first show how to use $\alg$ to construct, for each $k$, a learning algorithm $\alg^k$ solving \sing{k}{q_c}.
    Given a data set $X'$ of size $k$, $\alg^k$ samples a data set $\tilde{X}$ of $n$ entries, where $\tilde{X}\sim X\mid \Xone=X'$, and random variable $X$ is a data set sampled from the process generating data sets for \learnlong{q_c}.
    $\alg^k$ can perform this sampling, since $\tilde{X}\backslash X'$ will contain no samples from the same subpopulations as those that generated $X'$ and the per-subpopulation distributions are sampled i.i.d.\ from $q_c$.
    $\alg^k$ then simulates $\alg$ on $\tilde{X}$ and outputs model $\tilde{M}=\alg(\tilde{X})$.
    We have
    \begin{align}
        \Pr[\alg^k\text{ errs on \sing{k}{q_c}}] = \Pr[\alg\text{ errs on \learn{q_c}}\mid E_1, K=k].
            \label{eq:general_error_same}
    \end{align}
    Observe that, for all $k$, we have equality across the conditional distributions:
    \begin{align*}
        \Pr[X=x, M=m\mid K=k] &= \Pr[\tilde{X}=x, \tilde{M}=m\mid K=k] \\
        \Pr[\Xone=x_S, M=m\mid K=k] &= \Pr[X'=x_S, \tilde{M}=m\mid K=k].
    \end{align*}
    This implies that $I(\Xone;M\mid K) = \E_k \left[I(X'; \tilde{M} \mid K=k)\right]$.
    Since $\alg^k$ solves \sing{k}{q_c}, if we can bound the error of $\alg^k$ using Equation~\eqref{eq:general_error_same} we will be able to apply our lower bound $f_k(\cdot)$.

    To that end, we first bound the error of $\alg$ conditioned on the test sample coming from a singleton subpopulation.
    We have
    \begin{align}
        \epsilon &= \Pr[\text{$\alg$ errs on \learn{q_c}}] - \Pr[\text{$\Aopt$ errs on \learn{q_c}}] \label{eq:Acomp_switch} \\
            &= \Pr[\bar{E_1}]\bigl(\Pr[\text{$\alg$ errs on \learn{q_c}}\mid \bar{E_1}] - \Pr[\text{$\Acomp$ errs on \learn{q_c}}\mid \bar{E_1}] \bigr) \nonumber\\
            &\quad + \Pr[E_1]\left(\Pr[\text{$\alg$ errs on \learn{q_c}}\mid E_1] - \Pr[\text{$\Acomp$ errs on \learn{q_c}}\mid E_1] \right).\nonumber
    \end{align}
    Rearranging and substituting in we get
    \begin{align}
        \Pr[&\text{$\alg$ errs on \learn{q_c}}\mid E_1] - \Pr[\text{$\Acomp$ errs on \learn{q_c}}\mid E_1]\nonumber \\
            &\le \frac{1}{\Pr[E_1]} \Bigl(\epsilon + \Pr[\bar{E_1}]\bigl(\Pr[\text{$\Acomp$ errs on \learn{q_c}}\mid \bar{E_1}] \nonumber \\
            &\hspace{5.5cm} - \Pr[\text{$\alg$ errs on \learn{q_c}}\mid \bar{E_1}] \bigr)\Bigr) \nonumber \\
            &= \frac{\epsilon + \phi_1(q_c)}{\tau_1 \mu_1 n}. \label{eq:general_partial_central}
    \end{align}
    Note that in general $\phi_1(q_c)$ may be positive.\footnote{To see this, consider an algorithm that assigns a high prior probability to $\bar{E_1}$. 
    Conditioned on $\bar{E_1}$, then, this algorithm might have higher accuracy than the optimal algorithm.} 
    
    We now decompose the error over $k$.
    \begin{align*}
        \Pr[\alg &\text{ errs on \learn{q_c}}\mid E_1] - \Pr[\text{$\Acomp$ errs on \learn{q_c}}\mid E_1]\\
            &= \sum_{k=1}^\npops \Pr[K=k\mid E_1]\Bigl(\Pr[\text{$\alg$ errs on \learn{q_c}}\mid E_1, K=k]  \nonumber \\
            &\hspace{4cm}- \Pr[\text{$\Acomp$ errs on \learn{q_c}}\mid E_1, K=k]\Bigr) \\
            &= \sum_{k=1}^\npops \Pr[K=k\mid E_1]\Bigl(\Pr[\alg^k \text{ errs on \sing{k}{q_c}}]\\
            &\hspace{4cm}- \Pr[\text{$\Acomp$ errs on \learn{q_c}}\mid E_1, K=k]\Bigr)
    \end{align*}
    plugging in Equation \eqref{eq:general_error_same}.
    We wish to compare to the optimal error for \sing{k}{q_c}, so we add and subtract a term with $\Aopt^k$ denoting the optimal algorithm for \sing{k}{q_c}. 
    We have
    \begin{align*}
        \Pr[\alg &\text{ errs on \learn{q_c}}\mid E_1] - \Pr[\text{$\Acomp$ errs on \learn{q_c}}\mid E_1]\nonumber\\
            &= \sum_{k=1}^\npops \Pr[K=k\mid E_1]\Bigl(\Pr[\alg^k \text{ errs on \sing{k}{q_c}}]\\
            &\hspace{4cm}- \Pr[\alg_{OPT}^k \text{ errs on \sing{k}{q_c}}]\\
            &\hspace{4cm}+ \Pr[\alg_{OPT}^k \text{ errs on \sing{k}{q_c}}]\\
            &\hspace{4cm}- \Pr[\text{$\Acomp$ errs on \learn{q_c}}\mid E_1, K=k]\Bigr)\\
            &= \sum_{k=1}^\npops \Pr[K=k\mid E_1]\Bigl(\epsilon_k + \Pr[\alg_{OPT}^k \text{ errs on \sing{k}{q_c}}]\\
            &\hspace{4.5cm}- \Pr[\text{$\Acomp$ errs on \learn{q_c}}\mid E_1, K=k]\Bigr)\\
            &= \sum_{k=1}^\npops \Pr[K=k\mid E_1]\epsilon_k - \phi_2(q_c),
    \end{align*}
    defining $\epsilon_k$ as the suboptimality of $\alg^k$ on \sing{k}{q_c}.
    
    We combine with Equation \eqref{eq:general_partial_central} to get
    \begin{align*}
        \sum_{k=1}^\npops \Pr[K=k\mid E_1]\epsilon_k \le\frac{\epsilon + \phi_1(q_c)}{\tau_1 \mu_1 n} + \phi_2(q_c) \le\frac{\epsilon + \phi_1(q_c)+\phi_2(q_c)}{\tau_1 \mu_1 n},
    \end{align*}
    since $1\ge \tau_1\mu_1 n$.
    By Bayes rule, we have that
    \begin{align*}
        \Pr[K=k\mid E_1] = \frac{\Pr[E_1\mid K=k]\Pr[K=k]}{\Pr[E_1]} &= \frac{\tau_1 k \Pr[K=k]}{\tau_1\mu_1 n}\\ 
        &= \frac{ k \Pr[K=k]}{\mu_1 n},
    \end{align*}
    so
    \begin{align*}
        \frac{\E[k \epsilon_k]}{\mu_1 n} \le\frac{\epsilon + \phi_1(q_c)+\phi_2(q_c)}{\tau_1 \mu_1 n}.
    \end{align*}
    This establishes the first part of the lemma.
    For the second, we use the modification of Jensen's inequality in Lemma \ref{lemma:modified_jensen}.
    We have
     \begin{align*}
        I(X;M) \ge \E_k[kd\cdot g(\epsilon_k)] &\ge  \E_k[k]\cdot g\left(\frac{\E_k[k\epsilon_k]}{\E_k[k]}\right) \\
            &= \mu_1 n   \cdot g\left(\frac{\E_k[k\epsilon_k]}{\mu_1 n}\right) \\
            &\ge \mu_1 n  \cdot g\left(\frac{\epsilon + \phi_1(q_c)+\phi_2(q_c)}{\tau_1 \mu_1 n}\right).
    \end{align*}
\end{proof}

\subsection{Blueprint for Task-Specific Lower Bounds}\label{sec:blueprint}

For each task, we ultimately wish to lower bound $I(X;M\mid P)$.
By the chain rule and nonnegativity of mutual information, it suffices to lower bound the mutual information with the singletons:
\begin{align}
    I(X;M\mid P) &= I(\Xone, X\backslash \Xone, \Nsing; M \mid P) \nonumber \\
        &= I(\Nsing; M\mid P) + I(\Xone; M\mid \Nsing, P) + I(X\backslash \Xone ; M\mid \Nsing, \Xone, P) \nonumber \\
        &\ge I(\Xone; M\mid \Nsing, P). \label{eq:XS_is_lbd}
\end{align}
We write out the definition of mutual information and remove $P$ from the second term, using the fact that conditioning never increases entropy:
\begin{align*}
    I(\Xone; M\mid \Nsing, P) &= H(\Xone\mid K, P) - H(\Xone\mid M, K, P) \\
         &\ge H(\Xone\mid K, P) - H(\Xone\mid M, K).
\end{align*}
Adding and subtracting $I(\Xone; M\mid K)$ allows us to reach the lower bound $I(\Xone; M\mid \Nsing, P) \ge I(\Xone; M\mid K) - I(\Xone; P\mid K)$.
To lower bound $I(X;M\mid P)$, then, we will lower bound $I(\Xone; M\mid K)$ and upper bound $I(\Xone; P\mid K)$.

The latter is easily done for our tasks, since the distributions are straightforward.
Lower bounding $I(\Xone; M\mid K)$ requires more effort.
To apply our central reduction in Lemma~\ref{lemma:central_reduction} to a specific task, we need to calculate a number of quantities:
\begin{enumerate}
    \item Upper bounds on the error of optimal algorithms and lower bounds on the error of any algorithm.
    \item Upper bounds on $\phi_1(q_c)$ and $\phi_2(q_c)$.
    \item Lower bounds on mutual information for algorithms solving \sing{k}{q_c} for any $k$.
    This is the core of our task-specific proofs.
\end{enumerate}
Plugging these pieces into Lemma \ref{lemma:central_reduction} finishes the proof

\section{Next-Symbol Prediction}\label{sec:SC}

We present a simple sequence prediction task.
Among other applications, sequence prediction is a standard problem in natural language processing \citep{jurafsky2014speech}.

\subsection{Task Description and Main Result}

In this task, the data samples are binary strings of varying lengths with binary labels.
With each string we also associate a subpopulation identifier, so $\inputspace\times\mc{Y}=\left([N]\times \bigcup_{\ell=0}^{d-1} \{0,1\}^{\ell}\right)\times \{0,1\}$.
To instantiate the task in our framework we need only define the component distribution, from which the subpopulation distributions are drawn i.i.d.

\begin{definition}[$q_{NSP}$ Component Distribution]\label{def:qNSP}
    For each subpopulation $j$, we define its distribution over labeled examples via a reference string $c_j\in\{0,1\}^d$. 
    We draw $c_j \sim q_{NSP} = \mathrm{Uniform}\left(\{0,1\}^d\right)$.
    
    Examples from $j$ are tuples: the subpopulation identifier $j$ and noisy prefixes of $c_j$.
    To generate the prefix, we draw $\ell \in\{0,\ldots,d-1\}$ uniformly at random and produce $z=BSC_{\SCnoise/2}(c_j(1:\ell))$, where $\SCnoise<1$ is a fixed noise parameter.
    The label is a noisy version of the next bit: $y = \mathrm{BSC}_{\SCnoise/2}(c_j(\ell+1))$.
    \terminalbox
\end{definition}

The technical detail of pairing each example from subpopulation $j$ with an identifier ``$j$'' simplifies the analysis but is not crucial; for any noise level $\delta$, if $d$ is sufficiently large then the learning algorithm will be able to correctly distinguish subpopulations with high probability.

\begin{corollary}\label{cor:NSP_corollary_unif}
    Recall that  $\learnlong{q_{NSP}}$ (Definitions~\ref{def:learn} and \ref{def:qNSP}) is the next-symbol prediction task defined by parameters $N, n, \pi, d$ and $\delta$.
    Let $N=n$ and let $\pi$ be the single-item list $(1/n)$, so that the mixture over the $n$ clusters is uniform. 
    Fix $\delta< 1$ and let $d$ grow with $n$, possibly at different rates.
    Then
    \begin{enumerate}
 \item Let $S\subseteq [n]$ be the indices of singleton data points. $H(X_S|P) = \Omega(nd)$.
 
 \item 
 Any algorithm $\alg$ that is $\eps$-suboptimal on \learnlong{q_{NSP}}  for $\eps = o(1)$ satisfies
 \begin{equation}
    I(X;M\mid P) \ge (1-o(1))\cdot  H(X_S \mid P).
 \end{equation}
 \end{enumerate}
\end{corollary}

Recall that we interpret this result as stating that $\alg$ is forced to \textit{memorize whole samples}; as the error of the algorithm approaches optimal, $\alg$ must reveal almost all of the information about its singletons.

More generally, we prove the following theorem, which applies to any subpopulation/prior setup, as defined in terms of $N, n$, and $\pi$.
This form also exposes low-order terms and how the $\eps$-suboptimality affects the lower bounds.

\begin{theorem}\label{thm:NSP_main}
    Consider the problem \learnlong{q_{NSP}} specified by $d,\delta,\pi,N$ and $n$. 
    \begin{enumerate}
        \item Let $S\subseteq [n]$ be the indices of singleton data points. $H(X_S|P) \geq  \mu_1 \cdot n \cdot \frac{d+1}{2} \cdot h(\SCnoise/2)$.
     \item     Any algorithm $\alg$ that is $\eps$-suboptimal on \learnlong{q_{NSP}} satisfies
    \begin{align*}
        I(X; M\mid P)
        \ge \mu_1 \cdot n \cdot \frac{d+1}{2} \left(h(\SCnoise/2) - h\left(\frac{2 \epsilon}{\tau_1 \cdot \mu_1 \cdot n \cdot(1-\SCnoise)^2}\right)\right) 
        -n\log N.
    \end{align*}
    where $\mu_1, \tau_1$ depend on $\pi$ and $n$ as defined in Equations~\eqref{eq:tau1_def} and \eqref{eq:mu1_def}.
    \end{enumerate}
\end{theorem}

\begin{remark}\label{remark:NSP_easy_analysis}
    Unlike tasks based on clustering, one can analyze Next-Symbol Prediction by dealing with every subpopulation independently.
    This allows a similar but simpler proof than the one we present here.
    We use the same proof structure across tasks for consistency of presentation.
\end{remark}

In the rest of this section, we prove the main claims via the steps sketched in Section~\ref{sec:blueprint}: we show that the data contains limited information about the problem instance $P$, analyze the optimal error, and provide the central lower bound on the \sing{k}{q_{NSP}} task.

\subsection{Low Information about the Problem Instance}

\begin{lemma}\label{lemma:NSP_small_info_P}
    For \learn{q_{NSP}}, 
    \begin{align*}
        I(\Xone;P\mid K)\le \mu_1 n \cdot \frac{d+1}{2} \cdot (1 - h(\SCnoise/2)) + \mu_1 n \log N.
    \end{align*}
\end{lemma}
\begin{proof}
    We can think of $\Xone$ as being generated by first selecting $|\Xone|=k$ and then picking the subpopulations from which the singletons come.
    Write the subpopulation identifiers $\vec{J}\in [N]^{k}$, which needs at most $k\log N$ bits to describe.
    We have 
    \begin{align*}
        I(\Xone; P\mid K) &= I(\Xone, \vec{J}; P\mid K) \\
            &= I(\Xone; P\mid \vec{J}, K) + I(\vec{J}; P\mid K) \\
            &\le \E_k[I(\Xone; P\mid \vec{J}=\vec{j}, K=k)] + \E_k[k\log N],
    \end{align*}
    using the fact that $I(\Xone; P\mid \vec{J}=\vec{j}, K=k)$ does not depend on $\vec{j}$.

    With $\vec{j}$ and $k$ fixed, all subpopulations are independent, so let $X_1$ be a singleton and write
    $I(\Xone; P\mid \vec{J}=\vec{j}, K=k) = k\cdot I(X_1; C_1)$,
    using $C_1\sim q_{NSP}$ to denote the the reference string. 
    Note that $L_1$, the length of $X_1$, is fixed by $X_1$ and independent of $C_1$, so 
    \begin{align*}
        I(X_1; C_1) 
            = I(X_1, L_1 ; C_1) 
            = I(L_1 ; C_1) + I(X_1 ; C_1 \mid L_1) 
            = I(X_1 ; C_1\mid L_1).
    \end{align*}
    For each fixed $L_1=\ell$, $X_1$ (with its label) is just $\ell+1$ bits of $C_1$ run through a binary symmetric channel, 
    \begin{align*}
        I(X_1 ; C_1 \mid, L_1=\ell) = (\ell+1) \cdot( 1 - h(\SCnoise/2)).
    \end{align*}
    Putting together all the expectations and recalling that, by definition, $\E[k]=\mu_1 n$ finishes the proof.
\end{proof}

\subsection{Error Analysis}

\begin{proposition}[Accuracy of Optimal Algorithm]\label{prop:SC_accuracy}
    Let $E_0, E_1$, and $E_{>1}$, respectively be the events that the test sample comes from a subpopulation with zero, one, or multiple representatives in the data set. 
    Learning algorithm $\Aopt$ for \learnlong{q_{NSP}} achieves 
    \begin{enumerate}
        \item $\Pr[\text{$\Aopt$ errs} \mid E_1] = \frac{1}{2} - \frac{(1-\SCnoise)^2}{4} \le \frac{1}{4} + \frac{\SCnoise}{2}$.
        \item $\Pr[\text{$\Aopt$ errs} \mid E_{>1}] \le \Pr[\text{$\Aopt$ errs} \mid E_1]$.
        \item $\Pr[\text{$\Aopt$ errs} \mid E_{0}] \le \frac{1}{2}$.
    \end{enumerate}
\end{proposition}
\begin{proof}
    Each subpopulation can be dealt with independently.
    We will analyze the error conditioned on $E_1$; the error of the optimal algorithm conditioned on $E_{>1}$ can be no greater (since the algorithm can ignore samples), establishing (2), and we have $\Pr[\text{error}\mid E_0]=\frac{1}{2}$ for all algorithms, establishing (3).
    
    Condition on $E_1$, so Alice has one relevant string.
    Let $\ell_A$ denote the length of Alice's string and $\ell_B$ denote the length of Bob's string; Bob wants to output $c_{\ell_B+1}$.
    If $\ell_A\ge \ell_B$, he should output the $\ell_B+1$-th bit of her input, and otherwise answer randomly.
    Define a ``good event'' $G$ that occurs when Alice's input is longer than Bob's and neither her nor his $\ell_B+1$-th bit was rerandomized.
    (Bob doesn't recieve his $\ell_B+1$-th bit; it's the label his output is compared to.)
    We have
    \begin{align*}
        \Pr[G] = \left(\frac{1}{2} + \frac{1}{2d} \right) (1-\SCnoise)(1-\SCnoise) \ge \frac{1}{2}\cdot (1-\SCnoise)^2.
    \end{align*}
    Conditioned on $G$ this algorithm has error 0, and conditioned on $\bar{G}$ \emph{any} algorithm has accuracy $\frac{1}{2}$.
    So we have $\Pr[\text{$\Aopt$ errs}\mid E_1] = \frac{1}{2} \left(1 - \Pr[G]\right) \le \frac{1}{2} - \frac{(1-\SCnoise)^2}{4}$.
    Furthermore, this fact implies no algorithm can do better when $E_1$ occurs.
\end{proof}

\begin{proposition}\label{prop:NSP_phis}
    $\phi_1(q_{NSP}) \le 0$ and  $\phi_2(q_{NSP}) = 0$.
\end{proposition}
\begin{proof}
    We use the fact that the optimal strategy treats subpopulations independently, and thus is optimal for all $k$ and no matter which subpopulation the test sample comes from.
    Therefore
    \begin{align*}
        \phi_1(q_{NSP}) &= \Pr[\bar{E_1}]\Bigl(\Pr[\text{$\Aopt$ errs on \learn{q_{NSP}}}\mid \bar{E_1}] \\
             &\hspace{2.5cm}- \Pr[\text{$\alg$ errs on \learn{q_{NSP}}}\mid \bar{E_1}]\Bigr) \\
            &\le 1\times 0.
    \end{align*}
    For every $k$, the probability that $\Aopt$ errs on \learn{q_{NSP}} conditioned on $E_1$ and $K=k$ is exactly the optimal error on \sing{k}{q_{NSP}}, so
    \begin{align*}
        \phi_2(q_{NSP}) &= \sum_{k=1}^{\npops} \Pr[K=k\mid E_1]\Bigl( \Pr[\Aopt \text{ errs on \learn{q_{NSP}}}\mid E_1, K=k] \\
        &\hspace{4cm} - \inf_{\alg'} \Pr[\alg' \text{ errs on \sing{k}{q_{NSP}}}]\Bigr) \\
        &= \sum_{k=1}^{\npops} \Pr[K=k\mid E_1] \times 0.
    \end{align*}
\end{proof}

\subsection{Lower Bound for Singletons Task}

\begin{lemma}\label{lemma:SC_singletons_bound}
    Any algorithm $\alg$ that is $\eps_k$-suboptimal on \sing{\nsing}{q_{NSP}} satifies
    \begin{align*}
        I(X'; M)&\ge k\cdot \frac{d+1}{2} \left(1 - h\left(\frac{2 \epsilon_k}{(1-\SCnoise)^2}\right)\right).
    \end{align*}
\end{lemma}

Before proving this lemma, we provide a mutual information lower bound for algorithms solving just one instance of Next-Symbol Prediction.
The proof extends one appearing in \cite{bar2004information,feldman2014sample} for a communication complexity task called Augmented Index.

\begin{lemma}\label{lemma:SC_one_shot_second}
    Any algorithm $\alg$ that is $\eps_1$-suboptimal on ``one-shot'' \sing{1}{q_{NSP}} satisfies
    \begin{align*}
        I(X; M)&\ge \frac{d+1}{2} \left(1 - h\left(\frac{2 \epsilon_1}{(1-\SCnoise)^2}\right)\right).
    \end{align*}
\end{lemma}

\ifSTOC
    \begin{proof}
\else
    \begin{proof}[Proof of Lemma~\ref{lemma:SC_one_shot_second}]
\fi
    Let random variable $L_A$ be the length of Alice's input.
    Since we know $H(X)=H(X,L_A) = H(L_A)+H(X\mid L_A) = \log d + \frac{d+1}
    {2}$, to lower bound the mutual information we must provide an upper bound on $H(X\mid M)$.
    
    Define $G$ to be the ``good event'' that (i) Alice's input is at least as long as Bob's and (ii) the relevant bits were not rerandomized.
    $G$ happens with probability $\frac{1}{2}\cdot\frac{d+1}{d}(1-\delta)^2$.
    Conditioned on $G$, the optimal algorithm is correct and, conditioned on $\bar{G}$, \textit{any} algorithm has accuracy $\frac{1}{2}$.
    The main idea of the proof is that, conditioned on $G$, ``correctness'' and ``outputting Alice's data'' are the same event.
    
    We change the additive error $\epsilon_1$ into a multiplicative error $\gamma$.
    Let $\gamma \defeq \Pr[\text{$\alg$ errs} \mid G]$.
    We can write
    \begin{align*}
        \Pr[\text{$\alg$ errs}] &= \Pr[\text{$\alg$ errs} \mid G] \Pr[G] + \Pr[\text{$\alg$ errs} \mid \bar{G}] (1 - \Pr[G]) \\
        &= \frac{1}{2} - \frac{\Pr[G]}{2}\left( 1 - 2\gamma\right),
    \end{align*}
    and $\Pr[\text{$\alg_{OPT}$ errs}] = \frac{1}{2} - \frac{\Pr[G]}{2}$.
    By the definition of suboptimality we have $\epsilon_1 = \Pr[\text{$\alg$ errs}] -\Pr[\text{$\alg_{OPT}$ errs}] = \Pr[G] \cdot \gamma$.
    Since $\Pr[G]\ge \frac{1}{2}\cdot(1-\SCnoise)^2$, we have $\gamma \le \frac{2\epsilon_1}{(1-\SCnoise)^2}$, which implies $\gamma\le \frac{1}{2}$.
    
    Let random variables $X$ and $Z$ denote Alice and Bob's inputs, respectively, and write $L_A$ and $L_B$ for the lengths of their inputs.
    Since $L_A$ is fixed given $X$, we can apply the chain rule for entropy and bound
    \begin{align*}
        H(X \mid M) = H(X,L_A\mid M) &= H(X \mid M, L_A) + H(L_A\mid M) \\
            &\le \E_{\ell}[H(X \mid M, L_A=\ell)] + \log d.
    \end{align*}
    We will fix $\ell$ and bound $H(X \mid M, L_A=\ell)$.
    Define $\gamma_\ell$ to be $\Pr[\text{$\alg$ errs} \mid G, L_A=\ell]$.
    Assume $\gamma_{\ell} \le \frac{1}{2}$ without loss of generality; for any algorithm with a $\gamma_{\ell}> \frac{1}{2}$ there is one with the same information cost that achieves lower error by reversing the decision.
    
    Recall that $G$ implies neither Alice nor Bob's $L_B+1$-th bits are rerandomized, so Bob is correct if and only if he outputs Alice's $L_B+1$-th bit.
    We now show that, if Bob can output Alice's bits, her message must contain a lot of  information about her input.
    Crucially, conditioned on $L_A=\ell$, the good event $G$ is independent of Alice's input $X$, since $L_B\perp L_A$ and Alice's string is uniformly random whether or not any bit is flipped.
    So $H(X\mid M, L_A=\ell) = H(X\mid M, G, L_A=\ell)$.
    Below, in \eqref{eq:NSP_pf_1}, we apply the chain rule for entropy over Alice's bits, including her label.
    In \eqref{eq:NSP_pf_3} we replace $X_1^{\ell_B}$ with $Z_1^{\ell_B}$, which is a noisy version and thus can only increase uncertainty about $X_{\ell_B+1}$.
    \begin{align}
            H(X\mid M, G, L_A=\ell) &= \sum_{\ell_B=0}^{\ell} H(X_{\ell_B+1} \mid X_1^{\ell_B}, M, G, L_A=\ell) \label{eq:NSP_pf_1}\\
                &\le \sum_{\ell_B=0}^{\ell} H(X_{\ell_B+1} \mid Z_1^{\ell_B}, M, G, L_A=\ell). \label{eq:NSP_pf_3}
    \end{align}
    Now we relate the index $\ell_B$ to the random variable $L_B$, the length of Bob's input, and observe that $\Pr[L_B=\ell_B\mid G,L_A=\ell] = \frac{1}{\ell+1}$, since event $G$ requires that $L_A\ge L_B$.
    So we have
    \begin{align*}
        H(X\mid M, L_A=\ell)
            &\le (\ell+1) \sum_{\ell_B=0}^{\ell-1} \bigl(\Pr[L_B=\ell_B\mid G,L_A=\ell] \\
            &\hspace{1.7cm} \times  H(X_{\ell_B+1} \mid Z_1^{\ell_B},M,G,L_A=\ell)\bigr)\\
            &= (\ell+1) \cdot H(X_{L_B+1}\mid Z, M, G, L_A=\ell) \\
            &\le (\ell+1) \cdot h(\gamma_{\ell}),
    \end{align*}
    applying Fano's inequality and using the assumption that $\gamma_{\ell}\le \frac{1}{2}$, since this is exactly Bob's task and, conditioned on $G$ and $L_A=\ell$, he fails 
    with probability at most $\gamma_{\ell}$.
    
    \ifSTOC
\else
    We use the extension of Jensen's inequality in Lemma \ref{lemma:modified_jensen} to push the expectation inside the binary entropy function and get
    \fi
    \begin{align*}
        H(X\mid M, L_A)
            &\le \E[\ell+1] \cdot h\left(\frac{\E[(\ell+1)\gamma_{\ell}]}{\E[\ell+1]} \right) \\
            &= \frac{d+1}{2} \cdot h\left( \frac{2\cdot \E[(\ell+1)\gamma_{\ell}]}{d+1} \right) \\
            & =\frac{d+1}{2}\cdot h(\gamma).
    \end{align*}
    The last equality follows from the facts that $\gamma\cdot \Pr[G]=\E_\ell\left[\gamma_{\ell} \Pr[G\mid L_A=\ell]\right]$,
    $\Pr[G] = \frac{d+1}{2d}\cdot (1-\SCnoise)^2$, and $\Pr[G\mid L_A = \ell] = \frac{\ell+1}{d}\cdot (1-\SCnoise)^2$.
    Using $\gamma\le \frac{2\epsilon_1}{(1-\SCnoise)^2}$ gives us
    \begin{equation*}
        H(X\mid M) \le \frac{d+1}{2}\cdot h\left(\frac{2\eps_1}{(1-\delta)^2}\right) + \log d
    \end{equation*}
    which, combined with $H(X) = \frac{d+1}{2} + \log d$, finishes the proof.
\end{proof}

We can now prove the Next-Symbol Prediction lower bound for \sing{\nsing}{q_{NSP}}.
\begin{proof}[Proof of Lemma \ref{lemma:SC_singletons_bound}]
    Suppose algorithm $\alg$ is $\eps_k$-suboptimal on \sing{\nsing}{q_{NSP}}.
    Then, for each subpopulation $i\in [k]$, define its error above optimal $\epsilon_k^i$ so that $\E[\epsilon_k^i]=\epsilon_k$.
    By the same synthetic-dataset reduction used to prove Lemma \ref{lemma:central_reduction}, $\alg$ can be turned into $k$ algorithms $\{\alg_i\}$, each solving \sing{1}{q_{NSP}} to suboptimality $\{\epsilon_k^i\}$, and with mutual information 
    \begin{align*}
        I(\Xone; \alg(\Xone)) &\ge \sum_{i=1}^k I(X_i; \alg_i(X_i)) \\
            &\ge \sum_{i=1}^k \frac{d+1}{2} \left(1 - h\left(\frac{2\epsilon_k^i}{(1-\SCnoise)^2}\right)\right) \\
            &\ge k\cdot \frac{d+1}{2} \left(1 - h\left(\frac{2 \epsilon}{(1-\SCnoise)^2}\right)\right),
    \end{align*}
    writing the sum as an expectation over uniform $i\in [k]$ and applying Jensen's inequality.
\end{proof}

\subsection{Completing the Proof}\label{sec:NSP_completing}

\begin{proof}[Proof of Theorem \ref{thm:NSP_main}]
    By Lemma~\ref{lemma:NSP_small_info_P}, we get the lower bound on $H(\Xone\mid P)$ claimed in (1).

    For (2) we assume an $\eps$-suboptimal algorithm $\alg$ for \learn{q_{NSP}}.
    Our lower bound on \sing{k}{q_{NSP}} from Lemma \ref{lemma:SC_singletons_bound} gives us a mutual information lower bound of 
    \begin{align}
        f_k(\epsilon_k)  = k\cdot g(\epsilon_k) = k\cdot \frac{d+1}{2} \left(1 - h\left(\frac{2 \epsilon}{(1-\SCnoise)^2}\right)\right).
    \end{align}
    The function $1-h(\cdot)$ is convex and strictly decreasing for arguments less than $\frac{1}{2}$ so we have, upper bounding via Proposition \ref{prop:NSP_phis} both $\phi_1(q_{NSP})$ and $\phi_2(q_{NSP})$ with 0,
    \begin{align}
        I(\Xone; M\mid \Nsing) &\ge \mu_1 n \cdot \frac{d+1}{2} \left(1 - h\left(\frac{2 \epsilon}{\tau_1 \mu_1 n(1-\SCnoise)^2}\right)\right) .
    \end{align}
    In Lemma~\ref{lemma:NSP_small_info_P} we proved 
    $I(\Xone;P\mid K)\le \mu_1 n \cdot \frac{d+1}{2} \cdot (1 - h(\SCnoise/2)) + \mu_1 n \log N$, 
    so using the calculations in Section~\ref{sec:blueprint} we have
    \begin{align}
        I(X;M\mid P) &\ge I(\Xone; M\mid K) - I(\Xone; P\mid K) \\
            &\ge \mu_1 n \cdot \frac{d+1}{2} \left( h(\SCnoise/2) -
            h\left(\frac{2 \epsilon}{\tau_1 \mu_1 n (1-\SCnoise)^2}\right)\right) -\mu_1 n\log N.
    \end{align}
\end{proof}

\section{Hypercube Cluster Labeling}\label{sec:GHP}

\subsection{Task Description and Main Result}

In this task, the data samples are binary strings labeled with their subpopulation index, so $\inputspace\times\mc{Y}= \{0,1\}^d\times [\npops]$.
Unlike Next-Symbol Prediction, this distribution is ``noiseless:'' there is a small set of relevant features which are deterministically set, and only the irrelevant features are picked at random.

\begin{definition}[$\qHC$ Component Distribution]\label{def:qHC}
    For each subpopulation $j$, we define its distribution over labeled examples via a small number of fixed positions.
    For each cluster $j\in [\npops]$, $\qHC$ generates fixed indices as
    \begin{itemize}
        \item For each index $i\in [d]$, independently flip a coin that comes up heads w.p. $\rho$.
        \item If heads:
        \begin{itemize}
            \item Add $i$ to $\mc{I}_j$, the indices of fixed features.
            \item Flip a fair coin to set the value $b_j(i)$, the fixed feature value at that location.
        \end{itemize}
    \end{itemize}
    Samples from the subpopulation are uniform over the hypercube defined by the \emph{unfixed} indices:
    to generate a data point from subpopulation $j$, we let
    \begin{align*}
        z(i) = \begin{cases} b_j(i) & \text{if $i\in \mc{I}_j$} \\
            \mathrm{Bernoulli}(1/2) &\text{otherwise} \end{cases}.
    \end{align*}
    This point's label is $j$.
    \terminalbox
\end{definition}

Observe that, for any index $i\in[d]$ and strings $z_1,z_2$ from the same cluster, we have $\Pr[z_1(i) = z_2(i)] = \frac{1+\rho}{2}$.
This is identical to producing $z_2$ by running $z_1$ through a binary symmetric channel with parameter $\frac{1-\rho}{2}$, a fact which will be crucial in our proofs.
We set $\rho$ carefully to ensure that the task is difficult.

The distribution $q_{HC}$ defines a learning problem \learnlong{q_{HC}} specified by $d$ and $\rho$ (which determine $q_{HC}$), and $N,n$ and $\pi$ (which specify the mixture structure). The simplest instance of our main result is for the case where there are exactly $n$ subpopulations of equal weight.

\begin{corollary}\label{cor:HC_corollary_unif}
    Recall that  $\learnlong{q_{HC}}$ (Definitions~\ref{def:learn} and \ref{def:qHC}) is the hypercube cluster labeling task defined by parameters $N, n, \pi, d$ and $\rho$.
    Let $N=n$ and let $\pi$ be the single-item list $(1/n)$, so that the mixture over the $n$ clusters is uniform. Let $d$ grow with $n$ such that  $d\ge n^{0.1}$, and set $\rho = \sqrt{\frac{2\ln a \mu_1 n - \ln \ln n}{d}}$ for a constant $a>1$.
    Then
    \begin{enumerate}
        \item Let $S\subseteq [n]$ be the indices of singleton data points.     
        $H(X_S\mid P) = \Omega(nd)$.
        \item Any algorithm $\alg$ that is $\eps$-suboptimal on \learnlong{q_{HC}} for $\eps\le 0.1$ satisfies
        \begin{equation*}
            I(X;M\mid P) \ge \Omega( H(X_S\mid P)).
        \end{equation*}
    \end{enumerate}
\end{corollary}

More generally, we prove the following theorem, which applies to any subpopulation/prior setup, as defined in terms of $N, n$, and $\pi$.
This form also exposes leading constants, low-order terms, and how the $\eps$ suboptimality affects the lower bound. 

\begin{theorem}\label{thm:HC_main}
    Assume $n$ is sufficiently large and $d\ge n^{0.1}$.
    Set $\rho = \sqrt{\frac{2\ln a\mu_1 n - \ln \ln n}{d}}$ for any constant $a>1$, where $\mu_1$ depends on $N, n$, and $\pi$ as defined in Equation~\eqref{eq:tau1_def}.
    Consider the problem \learnlong{q_{HC}}.
    \begin{enumerate}
        \item Let $S\subseteq [n]$ be the indices of singleton data points. 
            $H(X_S\mid P) \ge \mu_1 (1-\rho) n d$.
            
        \item 
        There exist constant $c_a,\alpha>0$ such that, for any $\eps$-suboptimal algorithm $\alg$ solving \learnlong{\qHC}, we have
        \begin{equation*}
            I(X;M\mid P)\ge   
            \frac{\left(1-c_a - \frac{\eps + 2(\mu_1 n)^{-\alpha}}{\tau_1\mu_1 n} - o(1)\right) }{2\ln 2 + o(1)}\cdot \mu_1 n d -  2\E[|k-\mu_1 n| ]d - n d\rho -  n\log N.
        \end{equation*}
        The $o(1)$ expressions hide terms that are $O\left(\log^{-1}n\right)$ and depend only on $n$ and $a$.

    \end{enumerate}
\end{theorem}

\begin{remark}
    As we show in Propsition~\ref{prop:sing_constant_error}, setting $\rho$ to $\sqrt{\frac{2\ln a\mu_1 n - \ln \ln n}{d}}$ ensures the optimal algorithm for \sing{\mu_1 n}{q_{HC}} has constant error.
    This is not the only regime that forces memorization. 
    For instance, with $\rho=\sqrt{\frac{c\ln (\mu_1  n)}{d}}$ for a constant $c>2$ (in which case the optimal algorithm for \sing{\mu_1 n}{q_{HC}} has error going to 0 as $n$ grows), one can prove a lower bound almost identical to that in Theorem~\ref{thm:HC_main}.
    We choose to state Theorem~\ref{thm:HC_main} with a specific parameter regime because we conjecture a significantly stronger result in that regime. The conjecture and its implications are discussed in Section~\ref{sec:GH_conjecture}.
\end{remark}




\begin{remark}\label{rem:HC_small_representation}
    With exactly one sample from a subpopulation, our results imply that the learning algorithm is forced to memorize.
    With additional samples, the learner can leak much less information by sending a smaller model.
    If the learner has $\approx \log d$ samples from subpopulation $j$, it can learn $(\mc{I}_j, b_j)$ exactly with high probability, since unfixed indices have feature values selected independently and uniformly.
    Given knowledge of $(\mc{I}_j, b_j)$, the learning algorithm can send a subset $T_j\subseteq \mc{I}_j$ of size $O(\log (n/\eps))$ and still achieve high accuracy, since samples from other clusters will match all the features in $T_j$ with probability exponentially small in $|T_j|$.
\end{remark}

\subsubsection{A Stronger Result under a Communication Complexity Conjecture}\label{sec:GH_conjecture}

Our lower bound for this task relies on proving a one-way information complexity lower bound for a task that, on its face, is not clearly related to the learning task at hand.    
\begin{definition}[Nearest of $k$ Neighbors]
    Alice receives $k$ strings $x_1,\ldots, x_k \in \bit{d}$, drawn i.i.d.~from the uniform distribution.
    Bob receives a string $y\sim BSC_{\frac{1-\rho}{2}}(x_{j^*})$ for some index $j^*\in [k]$, also chosen uniformly at random.
    Alice and Bob succeed if Bob outputs $j^*$.
    \terminalbox
\end{definition}

For this task, our approach only yields a lower bound of $I(X;M)\ge \alpha \cdot kd$ for a constant $\alpha < 1$, even letting the error $\eps$ vanish.
This result, stated in Lemma~\ref{lemma:HC_communication_bound}, does not admit the interpretation of ``memorizing whole samples.''
We conjecture that $\alpha$ can be replaced by $1-o(1)$.
As evidence for this conjecture, in Appendix~\ref{sec:one_way_GHP} we prove such a one-way information complexity bound for a similar communication task, the Gap-Hamming problem.
The information complexity of Gap-Hamming is known to be $\Omega(d)$, but to the best of our knowledge no previous proofs provided the correct (at least for one-way communication) leading factor.

\begin{conjecture}\label{conj:HC_singletons_bound}
    There exists a concave function $\psi:[0,1]\to[0,1]$ with $\psi(\eps)\xrightarrow{\eps\to 0} 0$ such that if $k$ is sufficiently large, $d\ge k^{0.1}$, and $\rho=\sqrt{\frac{2\ln ak-\ln \ln k}{d}}$ for any $a>1$, then any $\eps$-suboptimal protocol for Nearest of $k$ Neighbors satisfies
    \begin{align*}
        I(X';M) &\ge (1- \psi(\eps) - o(1))\cdot kd.
    \end{align*}
    The $o(1)$ term holds for any $a>1$ and is $o(1)$ in $k$ and $d$.
\end{conjecture}

A proof of this conjecture would immediately demonstrate the same ``memorization of whole samples'' behavior that we proved inherent in Next-Symbol Prediction.
For simplicity, we state the following implication in the same form as Corollary~\ref{cor:HC_corollary_unif}, a lower bound for the uniform-mixture setting.

\begin{proposition}\label{prop:HC_conjecture_corollary}
    Assume Conjecture~\ref{conj:HC_singletons_bound} holds. 
    Consider the setup of Corollary~\ref{cor:HC_corollary_unif}, i.e.\ let $N=n$ and take the uniform mixture over the $n$ clusters. 
    Let $d$ grow with $n$ such that  $d\ge n^{0.1}$, and set $\rho = \sqrt{\frac{2\ln a \mu_1 n - \ln \ln n}{d}}$ for a constant $a>1$.
    Then
    \begin{itemize}
        \item $H(X_S\mid P) = \Omega(nd)$.
        \item Any algorithm $\alg$ that is $\eps$-suboptimal on \learnlong{q_{HC}} for $\eps=o(1)$ satisfies
        \begin{equation*}
            I(X;M\mid P) \ge (1 - o(1)) \cdot H(X_S\mid P).
        \end{equation*}
    \end{itemize}
\end{proposition}

\subsubsection{Organization of this Section}

In the rest of this section, we prove the main claims via the steps outlined in Section~\ref{sec:blueprint}, first showing that the data contains limited information about the problem instance $P$.
We then analyze the optimal error, ultimately bounding the $\phi_1(q_{HC})$ and $\phi_2(q_{HC})$ terms that appear in Lemma~\ref{lemma:central_reduction}, the central reduction.
We provide the central lower bound on the \sing{k}{q_{HC}} task and show that it implies similar lower bounds on \sing{k'}{q_{HC}} for any $k' = k+o(k)$.
These pieces allow us to immediately prove Theorem~\ref{thm:HC_main}.

\subsection{Low Information about the Problem Instance}

\begin{lemma}\label{lemma:HC_small_info_P}
    For \learn{q_{HC}}, 
        $I(\Xone;P\mid K)\le \mu_1 n \cdot \rho d + \mu_1 n \log N$.
\end{lemma}
The proof is similar to that of Lemma \ref{lemma:NSP_small_info_P}, the analogous bound for $q_{NSP}$.
\begin{proof}
    We can think of $\Xone$ as being generated by first picking $|\Xone|=k$ and then picking the subpopulations from which the singletons come.
    Write the labels as a vector $\vec{Y}\in [N]^{k}$, which needs at most $k\log N$ bits to describe.
    We have $I(\Xone; P\mid K) \le \E_k[k\cdot I(X_1; \mc{I}_1, b_1)] + \E_k[k\log N]$, using $X_1$ to denote a singleton sample.
    Since the unfixed bits are uniform and $|\mc{I}_1|$ is a random variable,
    \begin{align*}
        I(X_1; \mc{I}_1, b_1)  = I(X_1; |\mc{I}_1|, \mc{I}_1, B_1) 
            &= I(X_1;|\mc{I}_1|) + I(X_1; \mc{I}_1 B_1\mid|\mc{I}_1|) \\
            &= 0 +  H(X_1\mid |\mc{I}_1|) - H(X_1\mid B_1, \mc{I}_1, |\mc{I}_y|) \\
            &= d - (d - \E[|\mc{I}_1|]).
    \end{align*}
    Since $|\mc{I}_1|\sim\mathrm{Bin}(d, \rho)$, we have $I(X_1; \mc{I}_1, b_1) = \rho d$.
    
    Putting together the expectations and recalling that $\E[k]=\mu_1 n$ by definition, we are done.
\end{proof}



\subsection{Error Analysis}

Before analyzing the baseline error for $\mathrm{Learn}(q_{HC})$, we show that the $\rho$ we have selected causes the optimal algorithm for $\sing{k}{q_{HC}}$ to have constant error.
We first point out the optimal strategy for $\sing{k}{q_{HC}}$ is to select the nearest point; the proof requires simply writing out the posterior probability, which can be expressed in terms of Hamming distances.
\begin{proposition}\label{prop:bayes_for_sing}
    For any $k,d$, and $\rho$, the Bayes-optimal strategy for $\sing{k}{q_{HC}}$ has Alice send Bob all of her data and has Bob output the label of the example closest (in Hamming distance) to his test example.
\end{proposition}

\begin{proposition}\label{prop:sing_constant_error}
    For any $a>1$,  $k$ sufficiently large, and $d\ge k^{0.1}$, let $\rho = \sqrt{\frac{2\ln ak - \ln \ln k}{d}}$.
    The optimal algorithm for $\sing{k}{q_{HC}}$ has error $c_a+o(1)$, where constant $c_a$ depends only on $a$ and the second term is $o(1)$ in $k$ and $d$.
    In particular, $c_a$ is bounded away from $0$ and $\frac{1}{2}$.
\end{proposition}
The proof relies on a theorem of Littlewood~\cite{littlewood1969probability} giving both upper and lower bounds on the tails of binomial random variables (see \cite{ahle2017asymptotic} for exposition and the form we present).
\begin{lemma}[\cite{littlewood1969probability, ahle2017asymptotic}]\label{lemma:binomial_tail}
    Let $X\sim \mathrm{Bin}(d,1/2)$.
    For any $1\ll x \ll d^{1/4}$,
    \begin{equation}
        \Pr\left[X\le \frac{d}{2} - x\sqrt{\frac{d}{4}}\right] = \frac{1 + o(1)}{\sqrt{2\pi} x} \exp\left\{ -x^2/2\right\}.
    \end{equation}
\end{lemma}
\begin{proof}[Proof of Proposition~\ref{prop:sing_constant_error}]
    Let random variable $B = d_H(X_{j^*}, Z)$ be the Hamming distance to the correct answer's point, and let $B_1,\ldots, B_{k-1}$ be the distances to the other points.
    By Proposition~\ref{prop:bayes_for_sing}, the optimal algorithm for $\sing{k}{q_{HC}}$ will be correct if, for all $i$, $B< B_i$.
    It will be incorrect if $\exists i$ such that $ B_i< B$.
    To show that $A_{\mathrm{OPT}}$ has constant probability of error (namely, bounded away from $0$ and $1-1/k$), we will show that both of these events happen with constant probability.
    
    Note that $B> \E[B]=d\cdot\frac{1-\rho}{2}$ happens with constant probability (approximately $1/2$), and is independent of the event that $\min_i B_i \le \E[B]$.
    Therefore it suffices to show that $\Pr[\min_i B_i \le \E[B]]$ is neither too large nor too small.
    
    Lemma~\ref{lemma:binomial_tail} applies when $1\ll \sqrt{2\ln ak -\ln \ln k }\ll d^{1/4}$, which is clearly satisfied whenever $k$ is sufficiently large and $d\ge k^{0.1}$.
    Thus, for any $B_i$,
    \begin{align*}
        \Pr[B_i\le \E[B]] =\Pr\left[B_i \le  \frac{d}{2} - \frac{\rho d}{2} \right]
        &= \Pr\left[B_i \le  \frac{d}{2} - \sqrt{2\ln ak -\ln \ln k }\cdot \sqrt{\frac{d}{4}} \right]\\
        &= \frac{1+o(1)}{\sqrt{2\pi}} \cdot \frac{e^{-\left(\sqrt{2\ln ak -\ln \ln k }\right)^2/2}}{\sqrt{2\ln ak -\ln \ln k }} \\
        &= \frac{1+o(1)}{\sqrt{2\pi}} \cdot \frac{1}{\sqrt{2 - o(1)}}\cdot \frac{1}{ak}.
    \end{align*}
    Since the random variables are independent, we have 
    \begin{align*}
        \Pr[\min_i B_i \le \E[B]] &= 1 - \left(1 - \Pr[B_i \le\E[B]]\right)^{k-1} \\
            &= 1 - \left( 1 - \frac{1+o(1)}{a'k}\right)^{k-1}
    \end{align*}
    (for some constant $a'>1$), which is $c_a + o(1)$ where $c_a$ depends only on $a$ and $o(1)$ is in $k$ and $d$.
    In particular, we can see that that $c_a$ is bounded away from 0 and $\frac{1}{2}$ when $a>1$ and is constant.
    %
\end{proof}

We now show that, when the test sample comes from a subpopulation with no representatives in the data set, no algorithm can do better than random guessing.
This is trivial except for the fact that the number of such subpopulations is a random variable.
Analogous to $\tau_1$ and $\mu_1$, we define the following terms
\begin{align}
    \tau_0 \defeq \Pr[\text{$(z,y)$ comes from $j$} \mid \text{$X$ contains no samples from $j$}],
         \text{ and }
    \mu_0 \defeq  \frac{\E[K_0]}{n}.\label{eq:tau0_mu0_def}
\end{align}

\begin{proposition}[Error Lower Bound]\label{prop:HC_error_lb}
    Let $E_0$ be the event that the test sample comes from a subpopulation with no representatives in the data set.
    Every algorithm $\alg$ satisfies
    \begin{align*}
        \Pr[\alg\text{ errs on \learn{q_c}}\mid E_0] \ge 1 - \frac{1}{\mu_0 n}.
    \end{align*}
\end{proposition}
\begin{proof}
    For any fixed $K_0=k_0$, no algorithm can achieve error below $1 - \frac{1}{k_0}$, since all unrepresented subpopulations are equally likely.
    Decompose $\alg$'s  error over $K_0$:
    \begin{align*}
        \Pr[\alg\text{ errs on \learn{q_c}}\mid E_0] &= \sum_{k=0}^{n} \Pr[K_0=k_0\mid E_0]  Pr[\alg\text{ errs on \learn{q_c}}\mid E_0, K_0=k_0] \\
            &\ge \sum_{k=0}^{n} \Pr[K_0=k_0\mid E_0]  \left(1 - \frac{1}{k_0}\right).
    \end{align*}
    By Bayes' rule,
    \begin{align}
        \Pr[K_0=k_0\mid E_0] &= \frac{\Pr[E_0\mid K_0=k_0]\Pr[K_0=k_0]}{\Pr[E_0]} 
            = \frac{\tau_0 k_0 \Pr[K_0=k_0]}{\tau_0 \mu_0 n}.
    \end{align}
    Thus, since the probabilities sum to 1, 
    \begin{align}
        \Pr[\alg\text{ errs on \learn{q_c}}\mid E_0] &\ge  \sum_{k=0}^{n} \Pr[K_0=k_0\mid E_0] -  
            \sum_{k=0}^{n} \Pr[K_0=k_0] \frac{k_0}{\mu_0 n k_0} \\
            &= 1 - \frac{1}{\mu_0 n}.
    \end{align}
\end{proof}

In Proposition~\ref{prop:HC_phis} below, our analysis of the baseline algorithm requires that we control the probability of two examples from an incorrect subpopulation both falling close to the test example.
This is non-trivial, since the examples are not independent. 
We bound this probability with the following Small-Set Expansion Theorem:
\begin{lemma}[SSE Theorem, see \cite{odonnell2014analysis}]
    Let $X\in_{R}\{0,1\}^d$ and $X'\sim \mathrm{BSC}_{\frac{1-\rho}{2}}(X)$ for any $\rho$.
    For any set $A\subseteq \{0,1\}^d$, we have
    $\Pr[X,X'\in A] \le \left(|A| \cdot 2^{-d}\right)^{\frac{2}{1+\rho}}$.
\end{lemma}

\begin{lemma}\label{lemma:two_spurious_samples}
    Let $n$ be sufficiently large and let $d\ge n$.
    Let $X, Y\in \{0,1\}^d$ be independent uniform random variables and let $X'\sim \mathrm{BSC}_{\frac{1-\rho}{2}}(X)$ for $\rho=\sqrt{\frac{2\ln a\mu_1 n - \ln \ln n}{d}}$.
    There exists a constant $\alpha_1>0$ such that
    \begin{equation}
        \Pr\left[ \max \{d_H(X,Y), d_H(X',Y)\}\le \frac{d}{2} - \frac{3\rho d}{8} \right]\le  (\mu_1 n)^{-(1+\alpha_1)}.
    \end{equation}
\end{lemma}
\begin{proof}
    Set $\gamma=\frac{3}{4}$ and write $\tau=d\cdot \frac{(1-\gamma\rho)}{2}=\frac{d}{2}-\frac{3\rho d}{8}$.
    Once $Y=y$ is fixed, there is some ball $B(\tau)$ of points within distance $\tau$.
    We bound $\Pr[X, X'\in B(\tau)]$.
    (Note that this probability is independent of the value of the test example).
    Since $X$ is uniform and $X'\sim\mathrm{BSC}_{\frac{1-\rho}{2}}(X)$, we apply the Small Set Expansion theorem.
    First, note that by standard upper bounds on the volume of Hamming balls and the binary entropy function we have
    \begin{align*}
        |B(\tau)| = \left|B\left(d\cdot \frac{1- \gamma \rho }{2}\right)\right| 
            \le 2^{d\cdot h\left(\frac{1-\gamma \rho}{2}\right)} 
            \le 2^{d\left(1 - \frac{(\gamma\rho)^2}{2\ln 2}\right)}.
    \end{align*}
    Thus, applying the SSE,
    \begin{align}
        \Pr[X, X'\in B(\tau)] &\le \left(|B(\tau)|\cdot 2^{-d}\right)^{\frac{2}{1+\rho}}  \\
        &\le 2^{ -\frac{(\gamma\rho)^2 d}{2\ln 2} \cdot \frac{2}{1+\rho} } = e^{-\frac{\gamma^2 \rho^2 d}{1+\rho}}.
    \end{align}
    Plugging in the value of $\rho$ in the numerator, we have 
    \begin{align}
        \Pr[X, X'\in B(\tau)] &\le \exp\left\{-\frac{\gamma^2 d}{1+\rho}
        \left(\sqrt{\frac{2\ln a\mu_1 n - \ln \ln n}{d}} \right)^2  \right\} 
        = \left( \frac{\ln n}{(a\mu_1 n)^2} \right)^{\frac{\gamma^2}{1+\rho}}
    \end{align}
    For sufficiently large $n$ and $d$ we have both $\ln n \le (a\mu_1 n)^{0.1}$ and $\rho\le 0.01$, so this term is upper bounded by $(a\mu_1 n)^{-1.9 \gamma^2 / 1.01} \le (a\mu_1 n)^{-1.05}$.
    Since $a>1$, for a simpler upper bound we omit it in the final statement.
\end{proof}

\begin{proposition}\label{prop:HC_phis}
    For any $a>1$, sufficiently large $n$, and $d\ge n^{0.1}$, set $\rho = \sqrt{\frac{2\ln a\mu_1 n - \ln \ln n}{d}}$.
    There is a baseline algorithm $A^*$ such that, for some constant $\alpha>0$, $\phi_1(\qHC) \le (\mu_1 n)^{-\alpha}$ and  $\phi_2(\qHC) = (\mu_1 n)^{-\alpha}$.
\end{proposition}
\begin{proof}
    $A^*$ operates as follows: for any subpopulation which received more than 2 representatives, Alice randomly throws away all but 2 of them.
    Alice then sends this (possibly smaller) data set to Bob.
    Bob then performs the following steps.
    First, he checks if, among the subpopulations with two representatives,
    there are any with both examples within Hamming distance $\tau$ of the test example, for threshold $\tau=\frac{d}{2}-\frac{3\rho d}{8}$.
    If there is exactly one such subpopulation, Bob outputs its label.
    If there are multiple such subpopulations, he picks one arbitrarily.
    If there is no such subpopulation, Bob outputs the label of the singleton which is nearest to his test example.
    
    Before working directly with the $\phi$ terms, let us analyze the error of this algorithm.
    By Lemma~\ref{lemma:two_spurious_samples}, two samples from an incorrect subpopulation have probability $(\mu_1 n)^{-(1+\alpha_1)}$ of both being within distance $\tau$.
    By a union bound over the (at most) $n$ such subpopulations, this probability is $(\mu_1 n)^{-\alpha_1}$.
    
    Suppose the correct subpopulation has two representatives in the data.
    Let random variables $A$ and $A'$ be the Hamming distances from the test sample to these points.
    Via a Hoeffding bound~\cite{mitzenmacher2017probability}, we show that with high probability both these random variables are less than $\tau$.
    Recall that $A \sim \mathrm{Bin}(d,\frac{1-\rho}{2})$.
    \begin{align*}
        \Pr[(A > \tau) \cup (A' > \tau)] \le 2 \Pr[A > \tau]
        &=2\Pr\left[ A>\frac{d}{2}-\frac{3\rho d}{8} \right] \\
        &=2\Pr\left[ A-\frac{d}{2}+\frac{\rho d}{2}>\frac{\rho d}{8} \right] \\
        &= 2\Pr\left[ \frac{A}{d}-\mu > \frac{\rho}{8} \right] \\
        &\le 2\exp\left\{ - d \rho^2 / 32 \right\} \\
        &= 2\exp\left\{ -  \frac{d}{32}\left(\sqrt{\frac{2\ln an - \ln \ln n}{d}} \right)^2 \right\} \\
        &= 2\exp\left\{ - \left(2\ln a\mu_1 n - \ln \ln n \right)/32 \right\} \\
        &= 2\left( \frac{\ln n}{(an)^2} \right)^{1/32} \le (\mu_1 n)^{-\alpha_2}
    \end{align*}
    for some constant $\alpha_2$ and sufficiently large $n$.

    We now analyze $\phi_1$.
    Break up the terms across mutually exclusive events $\bar{E_1}=E_0\cup E_{>1}$.
    \begin{align*}
        \phi_1(\qHC) &= \Pr[\bar{E_1}]\Bigl(\Pr[\text{$A^*$ errs on \learn{\qHC}}\mid \bar{E_1}] \nonumber \\
             &\hspace{2.5cm}- \Pr[\text{$\alg$ errs on \learn{\qHC}}\mid \bar{E_1}]\Bigr) \nonumber \\
             &= \Pr[E_0]\Bigl(\Pr[\text{$A^*$ errs on \learn{\qHC}}\mid E_0]\\
             &\hspace{2.5cm}- \Pr[\text{$\alg$ errs on \learn{\qHC}}\mid E_0]\Bigr) \\
             &\hspace{0.5cm}+ \Pr[E_{>1}]\Bigl(\Pr[\text{$A^*$ errs on \learn{\qHC}}\mid E_{>1}] \\
             &\hspace{2.5cm}- \Pr[\text{$\alg$ errs on \learn{\qHC}}\mid E_{>1}]\Bigr) \\
             &\le 1\cdot \left(1 - \Pr[\text{$\alg$ errs on \learn{\qHC}}\mid E_0]\right)  \\
             &\hspace{0.5cm}+ 1\cdot\left(\Pr[\text{$A^*$ errs on \learn{\qHC}}\mid E_{>1}] - 0\right).
    \end{align*}
    By the error lower bound in Proposition~\ref{prop:HC_error_lb}, $\Pr[\text{$\alg$ errs on \learn{\qHC}}\mid E_0]\ge 1 - \frac{1}{\mu_0 n}$, where $\mu_0$, as defined in Equation~\eqref{eq:tau0_mu0_def}, is the expected number of subpopulations with no representative.
    And, as we calculated above, there exists a constant $\alpha>0$ such that $\Pr[\text{$A^*$ errs on \learn{\qHC}}\mid E_{>1}] \le (\mu_1 n)^{-\alpha}$.
    Thus, overall, we have $\phi_1(\qHC) \le O(n^{-1}) + (\mu_1 n)^{-\alpha}$. 
    The latter term dominates for sufficiently large $n$.
    
    
    For $\phi_2$, observe that, conditioned on $E_1$, with probability at least $1-(\mu_1 n)^{-\alpha}$, $A^*$ outputs the label of the singleton that is closest to the test example.
    By Proposition~\ref{prop:bayes_for_sing}, this is the optimal strategy, regardless of the number of singletons $k$ (when conditioning on $E_1$).
    Thus
    \begin{align*}
        \phi_2(\qHC) &= \sum_{k=1}^{\npops} \Pr[K=k\mid E_1]\Bigl( \Pr[A^* \text{ errs on \learn{\qHC}}\mid E_1, K=k] \\
        &\hspace{4cm} - \inf_{\alg'} \Pr[\alg' \text{ errs on \sing{k}{\qHC}}]\Bigr) \\
        &\le \sum_{k=1}^{\npops} \Pr[K=k\mid E_1]\times (\mu_1 n)^{-\alpha} \le (\mu_1 n)^{-\alpha},
    \end{align*}
    since the probabilities sum to one.
\end{proof}

\subsection{Lower Bound for Singletons Task}

\begin{figure*}
    \includegraphics[width=0.99\textwidth]{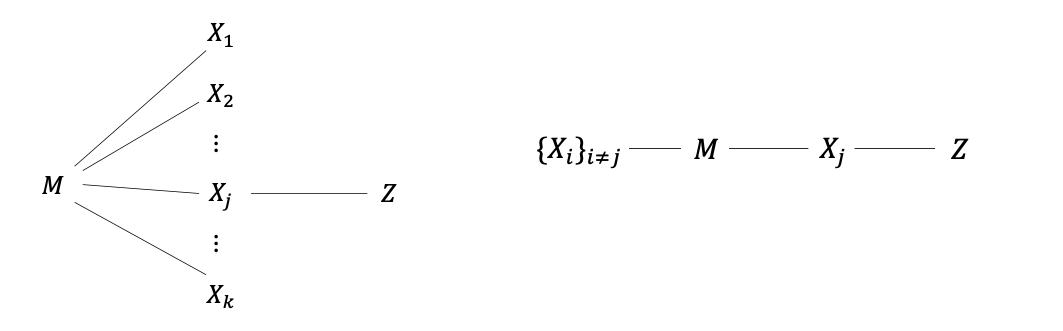}
    \caption{On the left, the dependency graph for \sing{k}{q_{HC}}, conditioned on $J=j$. On the right, the same graph rearranged to highlight the fact that $Z$ depends on $M$ only through $X_j$. This allows us to apply the strong data-processing inequality in our lower bound for the singletons task.}
    \label{fig:dependency_graph}
\end{figure*}

Our lower bound for \sing{k}{q_{HC}} is an immediate corollary of an identical lower bound on the external information complexity of Nearest of $k$ Neighbors.
To observe this, recall two properties of the task \sing{k}{q_{HC}}: 
(1) the joint distribution of any two samples from different subpopulations is the uniform product distribution, and
(2) the joint distribution of any two samples $z_1, z_2$ from the same subpopulation is uniform over $z_1$ with conditional distribution $BSC_{\frac{1-\rho}{2}}(z_1)$.

\begin{lemma}\label{lemma:HC_communication_bound}
    Set $\rho=\sqrt{\frac{2\ln ak - \ln \ln k}{d}}$ for constant $a>1$.
    Assume $k$ is sufficiently large and $d\ge k^{0.1}$.
    Any one-way communication protocol for Nearest of $k$ Neighbors with error at most $\eps_k$ satisfies
    \begin{align*}
        I(X' ;M) &\ge \frac{1 - c_a - \eps_k - o(1)}{2\ln 2 + o(1)}\cdot kd.
    \end{align*}
    The $o(1)$ expressions hide terms that are all $O\left(\log^{-1}n\right)$ and depend only on $k$ and $a$.
\end{lemma}

\begin{corollary}\label{cor:HC_singletons_bound}
    Set $\rho=\sqrt{\frac{2\ln ak - \ln \ln k}{d}}$ for constant $a>1$.
    Assume $k$ is sufficiently large and $d\ge k^{0.1}$.
    Any algorithm $\alg$ that is $\eps_k$-suboptimal on \sing{\nsing}{\qHC} is $\eps_k$-suboptimal on Nearest of $k$ Neighbors with the same information cost $I(X;M)$.
    Thus,  $\alg$ satisfies
    \begin{align*}
        I(X' ;M) &\ge\frac{1 - c_a - \eps_k - o(1)}{2\ln 2 + o(1)}\cdot kd.
    \end{align*}
\end{corollary}


The proof of Lemma \ref{lemma:HC_communication_bound} is adapted from one by Hadar et al. \cite{hadar2019communication} and relies on the following strong data processing inequality for binary symmetric channels.
\begin{lemma}[SDPI]\label{lemma:SDPI}
    Suppose we have a Markov chain $M - X - Y$ where $X\sim \mathrm{Uniform}(\bit{d})$ and $Y \sim \mathrm{BSC}_{\frac{1-\rho}{2}}(X)$.
    Then $I(M;Y) \le \rho^2 I(M;X)$.
\end{lemma}

\begin{proof}[Proof of Lemma~\ref{lemma:HC_communication_bound}]
    By Proposition~\ref{prop:sing_constant_error}, this value of $\rho$ results in the optimal algorithm having $c_a+o(1)$ error for some constant $c_a$ that depends only on $a$.
    Since Bob, with access to $M$ and test sample $Z$, can guess the index $J\in_R[k]$ with error at most $c_a + \epsilon_k+o(1)$, we have via Fano's inequality that
    \begin{align}
        I(J;M,Z) &= H(J) - H(J\mid M,Z) \nonumber \\
                &\ge \log k - ((c_a+\eps_k+o(1)) \log k + h(c_a + \eps_k)) \nonumber\\
                &\ge \left(1- c_a - \eps_k - o(1)\right)\log k, \label{eq:GHP_MI_lower_bd}
    \end{align}
    since $h(p)\le 1$ for all $p$.
    We now upper bound $I(J;M,Z)$.
    Let $P$ refer to the joint and marginal distributions defined by the learning task.
    We apply the ``radius'' property of mutual information and take $Q$ to be the product of marginals over $M$ and $Z$: $Q_{M,Z}=P_M\times P_Z$:
    \begin{align*}
        I(J;M,Z) &= \inf_{Q_{M,Z}\in \Delta((M,Z))}  \E_j\left[ D_{KL}\left( P_{M,Z\mid J=j} \Vert Q_{M,Z}\right)\right] \\
                &\le  \E_j\left[ D_{KL}\left( P_{M,Z\mid J=j} \Vert P_M\times P_Z\right)\right].
    \end{align*}
    Next, note that $M\perp J$ and $Z\perp J$, so we have 
    \begin{align*}
        \E_j\Bigl[ D_{KL}\bigl( &P_{M,Z\mid J=j} \Vert P_M\times P_Z\bigr)\Bigr] \\
            &= \E_j\left[ D_{KL}\left( P_{M,Z\mid J=j} \Vert P_{M\mid J=j}\times P_{Z\mid J=j}\right)\right] \\
            &=  \E_j\left[ I(M;Z\mid J=j) \right].
    \end{align*}
    
    Now we apply the SDPI. 
    \ifSTOC
\else
        For any fixed $j$, $M$ depends on the test sample $Z$ only through data point $X_j$.
        To illustrate this, observe that the left dependency graph drawn in Figure \ref{fig:dependency_graph} is equivalent to the Markov chain on the right.
    \fi
    We can marginalize out $\{X_i\}_{i\neq j}$ and apply Lemma \ref{lemma:SDPI}:
    \begin{align*}
        \E_j\left[ I(M;Z\mid J=j) \right] \le  \E_j\left[ \rho^2 I(M;X_j \mid J=j) \right].
    \end{align*}
    But for any index $i$, the mutual information between $M$ and $X_i$ is independent of $J$; it depends only on Alice's protocol.
    So 
    \ifSTOC
\else
        $\E_j\left[ \rho^2 I(M;X_j \mid J=j) \right] = \E_i\left[ \rho^2 I(M;X_i) \right]$
    \fi
    and, combining these steps and writing out the expectation, we have
    \begin{align}
        I(J;M,Z) \le \E_i\left[ \rho^2 I(M;X_i) \right] =  \frac{\rho^2}{k} \sum_{i=1}^k  I(M;X_i). \label{eq:GHP_MI_upper_bound}
    \end{align}

    Applying the chain rule for mutual information and the independence of the $\{X_i\}$, we get
    \begin{align*}
        \sum_i I(M;X_i) &= \sum_i H(X_i) - H(X_i \mid M) \\
                        &= \sum_i H(X_i \mid X_1^{i-1}) - H(X_i \mid M) \\
                        &\le \sum_i H(X_i \mid X_1^{i-1}) - H(X_i \mid M, X_1^{i-1}) \\
                        &= I(M;X).
    \end{align*}    
    Therefore, combining Equations \eqref{eq:GHP_MI_lower_bd} and \eqref{eq:GHP_MI_upper_bound}, 
    \begin{equation*}
        \left(1 - c_a - \eps_k -o(1)\right) \log k \le  I(J;M,Z) \le  \frac{\rho^2}{k} I(M;X).
    \end{equation*}
    Plugging in $\rho=\sqrt{\frac{2\ln ak - \ln \ln k}{d}}$ and changing the natural log to base 2, we see that
        $\frac{\rho^2}{k} = \frac{2 \ln 2 \cdot \log ak - \ln \ln k}{k d}$.
    Rearranging, we get a lower bound on $I(X;M)$.
\end{proof}

Our lower bound is for \sing{\nsing}{\qHC}, where $\rho$ is set in terms of $k$.
In learning task, though, the number of singletons is a random variable while $\rho$ remains fixed.
Here, we show that \textit{any} lower bound can be extended, with a slight loss in parameters, to ``misspecified'' tasks.

\begin{lemma}\label{lemma:other_sizes}
    Fix $\rho$.
    Suppose we have the following lower bound: there exists a function $g$ such that any algorithm $A$ that is $\eps$-suboptimal on \sing{k}{q_{HC}}
    satisfies $I(A(X);X) \ge g(\eps)\cdot kd$.
    Then, for any integer $t\in \mathbb{Z}$ and any algorithm $A'$, if $A'$ is 
    $\eps$-suboptimal on $\sing{k+t}{q_{HC}}$ then $A'$ satisfies
    \begin{equation*}
        I(A'(X);X) \ge g(\eps+2|t|/k)\cdot (k+t)d - |t|d.
    \end{equation*}
\end{lemma}
\begin{proof}
    In this proof we always take $t>0$ for legibility.
    We analyze the cases $k+t$ and $k-t$ separately.
    Use a superscript $X^{(k)}$ to denote data set size.
    In this proof, abbreviate \sing{k}{q_{HC}} as $\mathrm{Sing}(k)$.
    
    Take an algorithm $A'$ for $\mathrm{Sing}(k-t)$ with error $\mathrm{OPT}_{k-t}+\eps$.
    Construct an algorithm $A$ for $\mathrm{Sing}(k)$ as follows: $A$ gets an input $X^{(k)}$ and removes the last $t$ examples, generating a smaller data set $X^{(k-t)}$.
    $A$ then simulates $A'$ on this smaller data set.
    Let $E$ be the event that Bob's test example comes from one of the $k-t$ subpopulations represented in the smaller data set.
    Conditioned on this, by construction the error of $A$ is the same as that of $A'$.
    We thus have
    \begin{align*}
        \Pr[\text{$A$ errs on $\mathrm{Sing}(k)$}]
            &= \Pr[\text{$A$ errs on $\mathrm{Sing}(k)$}\mid E]\Pr[E]
            + \Pr[\text{$A$ errs on $\mathrm{Sing}(k)$}\mid \bar{E}]\Pr[\bar{E}] \\
            &\le \Pr[\text{$A$ errs on $\mathrm{Sing}(k)$}\mid E]\cdot 1 + 1\cdot \frac{t}{k} \\
            &= \Pr[\text{$A'$ errs on $\mathrm{Sing}(k-t)$}] + \frac{t}{k}\\
            &\le \mathrm{OPT}_{k-t} + \eps+ \frac{t}{k} \\
            &\le \mathrm{OPT}_{k} + \eps+ \frac{t}{k},
    \end{align*}
    where the last inequality follows from the fact that the error of the optimal algorithm, for a fixed $\rho$, increases with the number of subpopulations.
    By construction, $I(A'(X^{(k-t)});X^{(k-t)})=I(A(X^{(k)});X^{(k)})$.
    Since we have an upper bound on the error of $A$, by assumption we have a lower bound on its information cost.
    Thus we get
    \begin{align*}
        I(A'(X^{(k-t)});X^{(k-t)}) &= I(A(X^{(k)});X^{(k)}) \\
            &\ge g(\eps + t/k)\cdot kd \\
            &\ge g(\eps + t/k)\cdot (k-t)d.
    \end{align*}
    
    Now let $A'$ be an algorithm for $\mathrm{Sing}(k+t)$ (recall that $t>0$) with error at most $\mathrm{OPT}_{k+t}+\eps$.
    We construct an algorithm $A$ for $\mathrm{Sing}(k)$ as follows: $A$ receives the data set $X^{(k)}$, samples a dummy set $\tilde{X}$ of $t$ independent and uniform points, and runs $A'$ on $X^{(k)}\cup\tilde{X}$.
    Let $E'$ denote the event that the test example coming from one of the $k$ subpopulations that are in $X^{(k)}$.
    \begin{align*}
        \Pr[\text{$A'$ errs on $\mathrm{Sing}(k+t)$}]&= \Pr[\text{$A'$ errs on $\mathrm{Sing}(k+t)$}\mid E_2]\Pr[E_2] \\
        &\quad + \Pr[\text{$A'$ errs on $\mathrm{Sing}(k+t)$} \cap \bar{E}_2] \\
        &\ge  \Pr[\text{$A'$ errs on $\mathrm{Sing}(k+t)$}\mid E_2]\frac{k}{k+t}.
    \end{align*}
    Observe that the probability $A$ errs is exactly the probability $A'$ errs conditioned on $E'$, that is
        $\Pr[\text{$A$ errs on $\mathrm{Sing}(k)$}] = \Pr[\text{$A'$ errs on $\mathrm{Sing}(k+t)$}\mid E']$, so we have
    \begin{align*}
        \Pr[\text{$A$ errs on $\mathrm{Sing}(k)$}] &\le \frac{k+t}{k} \Pr[\text{$A'$ errs on $\mathrm{Sing}(k+t)$}]\\
        &\le \Pr[\text{$A'$ errs on $\mathrm{Sing}(k+t)$}] + \frac{t}{k} \\
        &\le \mathrm{OPT}_{k+t} + \eps + \frac{t}{k} \\
        &\le \mathrm{OPT}_{k} + \eps + \frac{2t}{k},
    \end{align*}
    where the last line follows from the fact that, for fixed $\rho$, increasing the data set size by $t$ can increase the optimal error by no more than $\frac{t}{k}$, as our calculations for the $k-t$ case show.
    So we have an upper bound on the error of $A$, and thus a lower bound on its information cost. 
    We want to turn this into a lower bound on the information cost of $A'$.
    \begin{align*}
        I(A'(X^{(k+t)});X^{(k+t)}) &= I(A'(X^{(k)}\cup \tilde{X});X^{(k)}, \tilde{X}) \\
        &= I(A'(X^{(k)}\cup \tilde{X});X^{(k)}) +  I(A'(X^{(k)}\cup \tilde{X});  \tilde{X} \mid X^{(k)})\\
        &\ge I(A'(X^{(k)}\cup \tilde{X});X^{(k)})\\
        &\ge I(A;X^{(k)}) \\
        &\ge g(\eps + 2t/k) \cdot kd \\
        &= g(\eps + 2t/k) \cdot kd \\
            &\quad + g(\eps + 2t/k) \cdot td \\
            &\quad - g(\eps + 2t/k) \cdot td \\
        &\ge g(\eps + 2t/k) \cdot (k+t)d - td,
    \end{align*}
    using the fact that $g(\cdot)\le 1$, which is without loss of generality since $1$ is the largest possible coefficient.
\end{proof}

\subsection{Completing the Proof}\label{sec:HC_completing}

\begin{proof}[Proof of Theorem \ref{thm:HC_main}]
    Part (1) follows from Lemma~\ref{lemma:HC_small_info_P}'s upper bound on $I(X;P\mid K)$.

    For part (2), we assume an algorithm $\alg$ for \learn{\qHC} with excess error $\eps$.
    Corollary~\ref{cor:HC_singletons_bound} gives a lower bound for \sing{\mu_1 n}{\qHC}, letting the number of singletons in that game be exactly the expected number over all.
    Lemma~\ref{lemma:other_sizes} extends this lower bound to other numbers of singletons. 
    Together, these steps give us
    \begin{equation*}
        f_k(\eps_k) = \frac{1 - c_a - \eps_k - |k - \mu_1 n|/\mu_1 n - o(1)}{2\ln 2 + o(1)}
        \cdot kd- |k- \mu_1 n| d,
    \end{equation*}
    where the $o(1)$ expressions both hide terms that are all $O\left(\log^{-1} n\right)$ and, crucially, do not depend on $k$.
    Our central reduction of Lemma~\ref{lemma:central_reduction} tells us that there exists a sequence of errors $\eps_k$ such that $I(\Xone; M\mid K)\ge \E_k[f_k(\eps_k)]$ and $\E_k[k\eps_k ]\le \frac{\eps + \phi_1 + \phi_2}{\tau_1}$. 
    Plugging these in yields, with a little bit of manipulation,
    \begin{align*}
        I(\Xone; M\mid K) &\ge \frac{(1-c_a -o(1))\E[k] d - \E[k\eps_k]d }{2\ln 2 + o(1)} - \frac{\E[|k-\mu_1 n| k]d}{(2\ln 2+o(1)) \mu_1 n} -  \E[|k-\mu_1 n| ]d \\
        &\ge \frac{(1-c_a -o(1))\mu_1 n d - \left(\frac{\eps + \phi_1 +\phi_2)}{\tau_1}\right)d }{2\ln 2 + o(1)} -  2\E[|k-\mu_1 n| ]d \\
        &\ge \frac{\left(1-c_a - \frac{\eps + 2(\mu_1 n)^{-\alpha}}{\tau_1\mu_1 n} - o(1)\right) }{2\ln 2 + o(1)}\cdot \mu_1 n d -  2\E[|k-\mu_1 n| ]d,
    \end{align*}
    using Lemma~\ref{prop:HC_phis} to upper bound $\phi_1$ and $\phi_2$.
    With Lemma~\ref{lemma:HC_small_info_P} to upper bound $I(\Xone;P\mid K)$, we prove part (2):
    \begin{align*}
        I(\Xone; M\mid P)
        &\ge I(\Xone; M\mid P,K)\\ 
        &\ge I(\Xone; M\mid K) - I(\Xone; P\mid K) \\
        &\ge \E_k[f(\eps_k)] -  \mu_1 n \cdot \rho d + \mu_1 n \log N. \\
            &= \frac{\left(1-c_a - \frac{\eps + 2(\mu_1 n)^{-\alpha}}{\tau_1\mu_1 n} - o(1)\right) }{2\ln 2 + o(1)}\cdot \mu_1 n d -  2\E[|k-\mu_1 n| ]d - n d\rho -  n\log N.
    \end{align*}
\end{proof} 

\begin{proof}[Proof Sketch for Corollary~\ref{cor:HC_corollary_unif}]
    We highlight the relevant details from the uniform mixture, many of which are also covered in Example~\ref{ex:uniform}.
    In the uniform setting with $N=n$, $\tau_1 = \frac{1}{n}$ and 
    \begin{equation*}
        \mu_1 = \left(1 - \frac{1}{n}\right)^{n-1} \approx \frac{1}{e}.
    \end{equation*}
    In this setting the number of singletons will concentrate (as can be shown via the Poisson approximation~\citep{mitzenmacher2017probability}), so $\E[|k-\mu_1 n|] = o(n)$.
    This implies that the negative terms in the lower bound are all $o(nd)$.
\end{proof}

\section{Experiments}\label{sec:experiments}

Our theorems are stated in terms of mutual information and do not explicitly address the question of efficient data reconstruction.
In this section, as a proof of concept, we present experiments exploring memorization and efficient black-box recovery.\footnote{Code available at \url{https://github.com/gavinrbrown1/training-data-memorization}.}
We generate synthetic data according to the Hypercube Cluster Labeling task and train multiclass logistic regression classifiers and single-hidden-layer feedforward neural networks to high accuracy.
We then attack the models: an adversary is given query access to the trained model and told the label of a singleton that appeared in the training data.

\paragraph{Data Generation}
We generate (synthetic) data sets for the hypercube cluster labeling task $\learnlong{q_{HC}}$ (Definitions~\ref{def:learn} and \ref{def:qHC}). 
We take $n=500$ examples and set $d=1000$ for the dimension.
We use the uniform-mixture setting considered in Corollary~\ref{cor:HC_corollary_unif}, setting $N=n$ and $\pi=(1/n)$, so that the data comes from the uniform mixture over the $N=500$ subpopulations. 
Recall that, in this task, each subpopulation is associated with a unique label, and the per-subpopulation distribution is specified by a set of fixed features, with the value of the feature fixed uniformly at random.
For each subpopulation, each feature is selected to be fixed independently with probability $\rho$.
We set $\rho \approx 0.17$, corresponding to $\rho=\sqrt{\frac{2\ln a\mu_1 n - \ln\ln n}{d}}$ with $a=50000$, a level at which the Bayes-optimal algorithm succeeds with almost perfect accuracy when the test example comes from a subpopulation which has a representative in the data.\footnote{We also ran experiments (not reported here) with $a=100$ and $\rho\approx 0.13$, which also result in near-perfect Bayes-optimal error. The larger value facilitates quicker training.}

\paragraph{Training and Hyperparameters}
Models were trained with PyTorch~\cite{pytorch}; all training algorithms referenced use that library's standard implementation.
We present results for (a) multiclass logistic regression (logit) classifiers and (b) single-hidden-layer feedforward neural networks (multilayer perceptron, or MLP) with 1500 hidden nodes and sigmoid activations.
Both models are trained via full-batch gradient descent with Nesterov momentum: logits for 50 gradient updates and MLPs for 2000 updates.
The training loss is standard cross-entropy.

The hyperparameters used were selected via a random search across a number of possible values.
The goal of the search was to locate high-accuracy (as measured by test-set classification error) settings; attacks were conducted after hyperparameters were selected.
The dimensions of the grid search included the optimization algorithm (among gradient descent with and without momentum, Adam, and Adagrad), learning rate, learning rate decay schedule, number of gradient updates, and width of the MLP. 

\paragraph{Model Evaluation}
We present various misclassification rates, collectively called ``classification error.''
The first two, ``train set'' and ``test set,'' are the standard misclassification rates on the training data set and a testing set of fresh samples, respectively.
Additionally, we define two metrics which make explicit use of the subpopulation structure:
\begin{itemize}
    \item \textit{Represented} error reports the classifier's misclassification rate on fresh examples drawn from subpopulations with at least one representative in the data.
    \item \textit{Singletons} error reports the classifier's misclassification rate on fresh examples drawn from subpopulations with exactly one representative in the data. 
\end{itemize}
We use these four metrics only to aid interpretion of the results.
In particular, they are not used to train models.
We include these metrics in Table~\ref{table:results} and use them in Figure~\ref{fig:over_training}.

\begin{algorithm}[H]\caption{Coordinate Ascent Attack}\label{alg:ascentattack}
\begin{algorithmic}[1]
\Require{classifier $f:\{0,1\}^d \to \Delta([N])$, target class $j^*$, number of iterations $T$}
\State $x \sim \mathrm{Uniform}(\{0,1\}^d)$
\For{$t = 1, \ldots, T$}
    \State $i \gets t \mod d$ 
    \If{$f(x^{i\mapsto 0})_{j^*} \ge f(x^{i\mapsto 1})_{j^*}$} \Comment{$x^{i\mapsto 0}$ sets the $i$-th bit of $x$ to $0$}
        \State $x_i \gets 0$
    \Else
        \State $x_i \gets 1$
    \EndIf
\EndFor
\State \Return $x$
\end{algorithmic}
\end{algorithm}

\begin{algorithm}[H]\caption{Gradient Sign Attack}\label{alg:signattack}
\begin{algorithmic}[1]
\Require{classifier $f:\{0,1\}^d \to \Delta([N])$, target class $j^*$, number of trials $k$}
\State $x\gets 0^d$ \Comment{store results}
\For{$i = 1, \ldots, d$}
    \State $\mathrm{count} \gets 0$ \Comment{count votes for ``0''}
    \For{$\ell = 1,\ldots, k$}
        \State $y \sim \mathrm{Uniform}(\{0,1\}^d)$
        \If{$f(y^{i\mapsto 0})_{j^*} \ge f(y^{i\mapsto 1})_{j^*}$} 
            \State $\mathrm{count} \gets \mathrm{count}+1$
        \EndIf
    \EndFor
    \If{$\mathrm{count} <  k/2$} 
        \State $x_i \gets 1$ \Comment{else keep $x_i=0$, as initialized}
    \EndIf
\EndFor
\State \Return $x$
\end{algorithmic}
\end{algorithm}

\paragraph{Attacks}
We present two simple attacks. 
Both require only black-box access to the classifier $f$ (returning a probability distribution over classes) and a target class $j^*$.
The computations are straightforward; in particular, no explicit inference is required.
Although more sophisticated algorithms might improve the results, we found these attacks sufficient for near-complete recovery in our settings.

The first attack, Algorithm~\ref{alg:ascentattack}, attempts to solve the problem $\max_x f(x)_{j^*}$, i.e. maximizing the probability of the target class.
The coordinate ascent algorithm picks a random starting location and iterates over indices, checking whether setting that index to 0 or 1 maximizes the objective.

Some classifiers (including the Bayes-optimal classifier) are not nicely behaved in a neighborhood of the singleton, and in these settings Algorithm~\ref{alg:ascentattack} often settles at an estimate that differs quite a bit from the true singleton.
As an alternative, we present Algorithm~\ref{alg:signattack}, inspired by the ``Fast Gradient Sign Attack'' introduced by Goodfellow et al.~\cite{goodfellow2014explaining} as a method to produce adversarial examples.
For each index $i\in [d]$, the attack randomly chooses $k$ strings and, on each, checks whether setting index $i$ to ``0'' or ``1'' maximizes the target-class probability.
As its guess for bit $i$, the attack outputs the majority vote from among the $k$ trials. 

\paragraph{Experiments and Results}
We run 20 independent trials, each consisting of generating a fresh problem instance and data set, training both logit and MLP classifiers, and executing both attacks on a randomly-chosen subset of singletons in the training data.
For each attack, the adversary receives a list of 20 labels, corresponding to 20 singletons in the data, and produces an estimate for each.
The adversary is evaluated on the percentage of bits they guess correctly, which we call ``recovery error.''
The results of the attacks, in addition to the measures of classification error, are summarized in Table~\ref{table:results}.

\begin{table}[h!]
\centering
\caption{Averages for classification error and recovery error. 
``Coordinate'' and ``Gradient'' report recovery error results for Algorithms~\ref{alg:ascentattack} and~\ref{alg:signattack}, respectively.}
\label{table:results}
\begin{tabular}{ l | c  c  c c | c c |}
    \cline{2-7}
    & \multicolumn{4}{|c|}{Classification Error (\%)} & \multicolumn{2}{|c|}{Recovery Error (\%)} \\ 
    \cline{2-7}
    & Train Set & Test Set &  Represented & Singletons & Coordinate & Gradient \\
    \hline
    \multicolumn{1}{|c|}{Logit} & 0.0 & 36.9 & 1.3 & 2.7  & 0.0 & 33.1 \\
    \multicolumn{1}{|c|}{MLP} & 0.0 & 38.4 & 3.4 & 6.0 & 6.0 & 2.3 \\
    \hline
\end{tabular}
\end{table}

Since our models are trained with iterative algorithms, it is natural to track how the adversary's success evolves during the training process.
Figure~\ref{fig:over_training} shows this for (single training runs of) the classifiers we consider, presented with measures of error.
As we can see, the attacks continually become more successful over time, even when (as in the case for the MLP) the classification errors are extremely non-monotonic.

One striking feature of the MLP results in Figure~\ref{fig:over_training} deserves further discussion.
Observe that, for the first roughly 750 gradient updates, the classification error on fresh samples from the singletons subpopulations (blue dotted line) is almost 100\%.
However, the attack recovers almost 90\% of the singleton bits (red solid line) after 750 gradient updates.
Even though the model has not yet ``learned what to do with the singletons'' in terms of classification, it has memorized a substantial amount of information about them.

\begin{figure}
    \centering
    \begin{subfigure}{0.49\textwidth}
        \includegraphics[width=\textwidth]{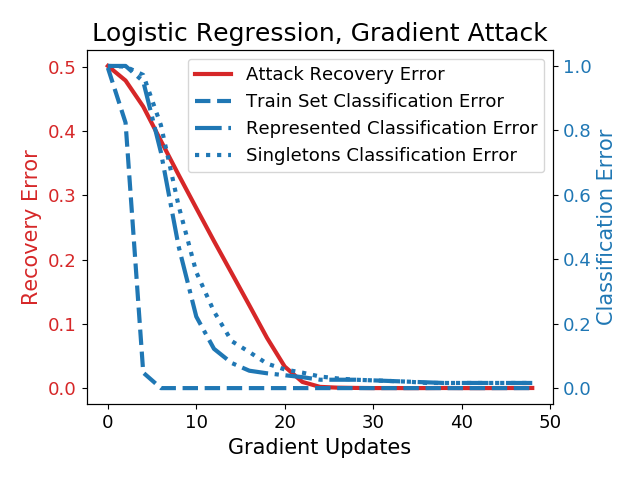}
    \end{subfigure}
    \begin{subfigure}{0.49\textwidth}
        \includegraphics[width=\textwidth]{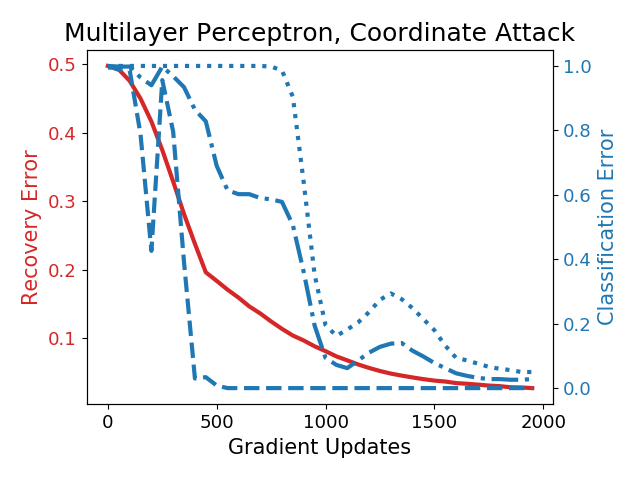}
    \end{subfigure}
    \caption{(a) Gradient attack recovery error over training iterations for logistic regression, plotted with classification error on the train set, fresh examples from represented subpopulations, and fresh examples from singleton subpopulation. 
    (b) The same plot for the multilayer perceptron and coordinate attack.
    The legend is shared across plots.}
    \label{fig:over_training}
\end{figure}

\paragraph{Discussion of Experiments}
Our lower bounds are in terms of mutual information, and do not guarantee that an efficient adversary can conduct data-reconstruction attacks.
Thus, as a proof of concept, our experiments complement the theoretical results: not only is the mutual information large, but simple attacks can succeed against popular learning algorithms.
The classifiers we evaluate, multiclass logistic regression and multilayer perceptron, are not designed to explicitly memorize whole training points, but do exactly that when trained to high accuracy on our hypercube cluster labeling task.
Our experiments suggest that avoiding such natural attacks requires, at the very least, intentional care on the part of the algorithm designer.

\appendix
\section*{Appendix}\addcontentsline{toc}{section}{Appendix}
\section{Additional Technical Details and Proofs}\label{sec:technical_details}

In this section, we present additional statements used in the paper.
We discuss the exact process for generating subpopulation mixtures, as in \citep{feldman2020does}.
We then prove the intuition of the bimodal prior in Example~\ref{ex:bimodal}.
We provide a worst-case version of our lower bound.
We prove an extension of Jensen's inequality which is used several times in the paper, and finally provide a version of our ``central reduction'' from the learning task to the singletons task which allows us to make use of a near-optimal but simpler-to-analyze baseline algorithm. 

\subsection{Generating Subpopulation Mixture Coefficients}\label{sec:long_tail_details}

We generate a mixture over $N$ subpopulations using the process introduced in \citep{feldman2020does}.
We begin with a list $\pi$ of nonnegative values.
For each subpopulation $j$, we sample a value $\delta_j \sim \mathrm{Uniform}(\pi)$.
To create a probability distribution $D$, we normalize:
\begin{align*}
    D(j) = \frac{\delta_j}{\sum_{i\in [N]} \delta_i}.
\end{align*}
This process is identical for all $j$, so we define $\bar{\pi}^N$ as the resulting marginal distribution over the mixture coefficient for any single subpopulation.

The quantity $\tau_1$ is defined in \citep{feldman2020does} as
\begin{align*}
    \tau_1 = \frac{\E_{\alpha \sim \bar{\pi}^N}\left[ \alpha^2 (1 - \alpha)^{n-1}\right]}{\E_{\alpha \sim \bar{\pi}^N}\left[ \alpha (1 - \alpha)^{n-1}\right]}.
\end{align*}
Lemma 2.1 of \citep{feldman2020does} proves the equality
\begin{align*}
    \E_{\substack{D\sim \mc{D}_{\pi}^N, \\ ID\sim D}} \left[D(j) \mid ID=id \right] = \tau_1.
\end{align*}
Observe that, for any set of cluster identifiers $ID=id$,
\begin{align*}
    \E_{\substack{D\sim \mc{D}_{\pi}^N, \\ ID\sim D}} \left[D(j) \mid ID=id \right] 
        &= \sum_{\alpha} \alpha \cdot \Pr[D(j) = \alpha\mid ID=id] \\
        &= \sum_{\alpha} \Pr[\iota(y) = j \mid D(j)=\alpha] \cdot \Pr[D(j) = \alpha\mid ID=id] \\
        &= \Pr[\iota(y) = j \mid ID=id].
\end{align*}

\subsection{Details for Bimodal Prior}

Here we provide the details necessary for Example \ref{ex:bimodal}.
Recall that we set $N=2^n$.
To build $\pi$ we add $1$ copy of $\frac{1}{2n}$ and $n2^{n} - 1$ copies of $\frac{1}{2\cdot 2^n}$.
This yields
\begin{align}
    \mathrm{Uniform}(\pi) = \begin{cases} \frac{1}{2n} & \text{w.p. $n2^{-n}$} \\
                    \frac{1}{2\cdot 2^{n}} & \text{w.p. $1 - n2^{-n}$}. \end{cases}.
\end{align}
We now show that the normalizing constant $C$ will concentrate about its mean.
Let $H$ be the number of heavy bins that are drawn.
We have, as a lower bound,
\begin{align}
    C = \frac{H}{2n} + \frac{1}{2\cdot 2^n}\left(2^n - H\right) = \frac{1}{2} + \frac{H}{2}\left(\frac{1}{n} - \frac{1}{2^n}\right) \ge \frac{1}{2}.
\end{align}
If $H\le 2n$ then $C\le \frac{3}{2}$, so by a Chernoff bound we have
\begin{align*}
    \Pr[C\ge 3/2] \le \Pr[H \ge 2\E[H]] \le e^{-n/3}.
\end{align*}

We can now lower bound $\mu_1$ and $\tau_1$ with results from \citep{feldman2020does}.
Define the weight after normalization.
\begin{align}
    \texttt{weight}\left(\bar{\pi}^N,\left[\beta_1,\beta_2\right]\right) \defeq 
        N\cdot \E_{\alpha\sim\bar{\pi}^N}\left[\alpha \cdot \mathbf{1}_{\alpha\in [\beta_1,\beta_2]}\right].
\end{align}
Equation 5 from \citep{feldman2020does} gives us that
\begin{align}
    \mu_1 n \ge \frac{n}{3} \texttt{weight}\left(\bar{\pi}^N,\left[0,\frac{1}{n}\right]\right) = \frac{n}{3},
\end{align}
since $C\ge \frac{1}{2}$ implies $\max \alpha \le \frac{1}{n}$.
To bound $\tau_1$ we use Lemma 2.5 from \citep{feldman2020does}:
\begin{align}
    \tau_1 \ge \frac{1}{5n} \texttt{weight}\left(\bar{\pi}^N,\left[\frac{1}{3n},\frac{2}{n}\right]\right).
\end{align}
Observe that we always draw $\alpha\ge \frac{1}{3n}$ when (i) we draw a heavy bin and (ii) $C \le \frac{3}{2}$, which always happens when $H\le 2n$.
By another Chernoff bound,
\begin{align*}
    \Pr[H\ge 2n \text{ or } H\le n/2] \le \Pr[H\ge 2n] + \Pr[H\le n/2] \le e^{-n/3} + e^{-n/8}.
\end{align*}
Since the probability of drawing a heavy bin is $n2^{-n}$, we have
\begin{align}
    \Pr_{\alpha\sim\bar{\pi}^N}\left[\alpha \ge \frac{1}{3n}\right]\ge n2^{-n} \left(1 - 2e^{-\Omega(n)}\right).
\end{align}
Thus, again applying $\max \alpha \le \frac{1}{n}$,
\begin{align}
    \tau_1 \ge \frac{1}{5n} \texttt{weight}\left(\bar{\pi}^N,\left[\frac{1}{3n},\frac{2}{n}\right]\right)
        &=  \frac{1}{5n}\cdot 2^n \cdot \E_{\alpha\sim\bar{\pi}^N}\left[\alpha \cdot \mathbf{1}_{\alpha\in [1/3n,2/n]}\right] \\
        &\ge \frac{2^{n}}{5n} \cdot \frac{1}{3n} \cdot \Pr_{\alpha\sim\bar{\pi}^N}\left[\alpha \ge \frac{1}{3n}\right] \\
        &\ge \frac{2^{n}}{5n} \cdot \frac{1}{3n} \cdot n2^{-n}\left( 1- e^{-\Omega(n)}\right).
\end{align}
Simplifying, we see that $\tau_1=\Omega(1/n)$.

\subsection{From Average-Case to Worst-Case via Minimax}\label{sec:minimax}

We provide metadistributions (dependent on the sample size $n$ and dimension $d$) upon which any near-optimal algorithm has high information cost.
A metadistribution is over problem instances, each of which is itself a distribution over labeled examples.
The ``easy'' direction of von Neumann's Minimax Theorem allows us to turn this into a worst-case guarantee. 

\begin{proposition}\label{prop:worst_case}
    Suppose $q^*$ is a metadistribution such that any algorithm that is $\eps$-suboptimal for Learn($n,N,q^*,\pi$) satisfies $I(X;M\mid P)=\Omega(nd)$. Then, for any $\alg$ that is $\eps$-suboptimal, there exists a problem instance $p^*\in \Delta(\mc{X})$ such that
    \begin{align*}
        I(X;M \mid P=p^*) \geq I(X;M|P) = \Omega(nd).
    \end{align*}
\end{proposition}
\begin{proof}
    Since $I(X;M|P)=\E_{p\sim q^*}[I(X;M\mid P=p)]$ is an expectation, there always exists a problem instance whose value is at least that of the expectation.
\end{proof}

The assumption of Propostion~\ref{prop:worst_case} deserves some discussion.
Specifically, recall that a learning algorithm $\alg$ is $\eps$-suboptimal for Learn($n,N,q_{n,d},\pi$) if it competes with the best possible learner $\Aopt^{(n,d)}$ for  $q_{n,d}$ when receiving a sample of size $n$. An equivalent requirement is that the error of $\alg$ on $p$ be close that that of $\Aopt^{(n,d)}$ on average over instances $p\sim q_{n,d}$. On one hand,  this assumption is much weaker than assuming that $\alg$ is near-optimal for each instance $p$ (since we compare with the learner $\Aopt^{(n,d)}$ whose average error over $p\sim q_{n,d}$ is lowest, not a learner that is tailored to $p$). On the other hand, it is not obviously comparable to the requirements of properness and consistency made by \citet{bassily2018learners,nachum2018direct}. 
For one thing, not all of our learning tasks fit neatly into the framework of PAC learning.
For those that do, the sample size $n$ that we consider is generally less than (or comparable to) the VC-dimension of the underlying concept class---too low to guarantee that every proper and consistent learner has high accuracy. Understanding the full relationship between these different kinds of assumptions is a subject for future work.

\subsection{Expectation Trick for Jensen's Inequality}

We will several times make use of the following application of Jensen's inequality.
Note that both inequalities apply in the other direction to convex functions.
\begin{lemma}\label{lemma:modified_jensen}
    Suppose $f$ is a concave function, $g$ is a function, and $X$ is a nonnegative random variable with distribution $p(x)$.
    Then
    \begin{align*}
        \E\left[X\cdot  f(g(X))\right] \le \E[X] f\left(\frac{\E[X\cdot g(X)]}{\E[X]}\right).
    \end{align*}
\end{lemma}
\begin{proof}
    We cannot directly apply Jensen's inequality, since $x\cdot f(x)$ will not in general be concave.
    Instead we introduce a distribution $q\propto x\cdot  p(x)$.
    Note that $\sum_x x \cdot p(x) = \E_p[X]$, so we have
    \begin{align*}
        q(x) = \frac{x \cdot p(x)}{\E_p[X]}.
    \end{align*}
    We now rewrite so that the expectation is over $q$:
    \begin{align*}
        \E_p[X\cdot f(g(X))] &= \sum_x p(x) \cdot x \cdot f(g(x)) \\
            &= \frac{\E_p[X]}{\E_p[X]} \sum_x p(x) \cdot x \cdot f(g(x)) \\
            &= \E_p[X] \sum_x q(x) \cdot f(g(x)) \\
            &= \E_p[X] \cdot \E_q[f(g(X))].
    \end{align*}
    We now have the expectation of a concave function and can apply Jensen's inequality and finish the proof:
    \begin{align*}
        \E_p[X\cdot f(X)] &\le \E_p[X] \cdot f\left(\E_q[g(X)]\right) \\
            &= \E_p[X] \cdot f\left(\sum_x \frac{ p(x)\cdot c \cdot g(x)}{\E_p[X]} \right) \\
            &= \E_p[X] \cdot f\left(\frac{\E_p[X\cdot g(X)]}{\E_p[X]}\right).
    \end{align*}
\end{proof}

\subsection{Central Reduction via Non-Optimal Baseline}

In some tasks the exactly-optimal algorithm $\Aopt$ may be difficult to analyze, while a naive ``baseline'' algorithm $\alg^*$ exists with error approaching optimal.
To deal with this, we will use the following modification of Lemma \ref{lemma:central_reduction}, where the error terms $\phi_1$ and $\phi_2$ now depend on $\alg^*$.
The proof is almost line-by-line identical to that of Lemma \ref{lemma:central_reduction}, observing in Equation \eqref{eq:Acomp_switch} that
\begin{align}
    \epsilon &= \Pr[\text{$\alg$ errs on \learn{q_c}}] - \Pr[\text{$\Aopt$ errs on \learn{q_c}}] \\
        &\ge \Pr[\text{$\alg$ errs on \learn{q_c}}] - \Pr[\text{$\alg^*$ errs on \learn{q_c}}]
\end{align}
and proceeding with $\alg^*$ in place of $\Aopt$.

\begin{lemma}[Central Reduction]\label{lemma:central_reduction2}
    Suppose we have the following lower bound for every $k$: any algorithm
    $\alg^k(X')$ 
    that is $\eps_k$-suboptimal for \sing{\nsing}{q_c} satisfies 
    \begin{align*}
        I(X';\alg^k(X')\mid K=k) \ge f_k(\epsilon_k).
    \end{align*}
    
    For any algorithm $\alg(X)$ that is $\eps$-suboptimal on \learn{q_c}, there exists a sequence $\{\epsilon_k\}_{k=1}^n$ such that $\E_k[k\eps_k] = \eps$ and $I(X_S;M\mid K) \ge \E_k[f_k(\epsilon_k)]$.
    
    Furthermore, if $f_k(\epsilon_k) \ge k \cdot g(\epsilon_k)$ for convex and nonincreasing $g(\cdot)$, then
    \begin{align*}
        I(\Xone;M\mid K) &\ge \mu_1 n \cdot g\left(\frac{\epsilon + \phi_1(q_c) + \phi_2(q_c)}{\tau_1 \mu_1 n}\right).
    \end{align*}
    Here $\phi_1(q_c) $ and $ \phi_2(q_c)$ are task-specific terms, defined by
    \begin{align}
        \phi_1(q_c) &= \Pr[\bar{E_1}]\left(\Pr[\text{$\alg^*$ errs on \learn{q_c}}\mid \bar{E_1}] 
             - \Pr[\text{$\alg$ errs on \learn{q_c}}\mid \bar{E_1}]\right) \, , \\
        \phi_2(q_c) &= \sum_{k=1}^{\npops} \Pr[K=k\mid E_1]\Bigl( \Pr[\alg^* \text{ errs on \learn{q_c}}\mid E_1, K=k] \\
        &\hspace{4cm} - \inf_{\alg'} \Pr[\alg' \text{ errs on \sing{k}{q_c}}]\Bigr).
    \end{align}
    $E_1$ is the event that the test sample comes from a subpopulation with exactly one representative.
\end{lemma}

\section{Tasks Related to Next-Symbol Prediction}\label{sec:additional_models}

\subsection{Lower Bound for Threshold Learning}

\newcommand{\T}{\mc{T}_d}

Threshold learning is a simple and well-studied binary classification task.
Data $z\in\bit{d}$ receives a label according to threshold $c\in\bit{d}$, so $y=0$ if $c\ge z$ and $y=1$ otherwise. 

Via a reduction from Next-Symbol Prediction, we demonstrate memorization in this setting.
This substantially strengthens results in \citep{bassily2018learners,nachum2018direct}, which contained bounds of $\Omega(\log d)$ for any proper, consistent learner; we prove a lower bound of $\Omega(d)$ while allowing any $\eps$-suboptimal learner.
As in Next-Symbol Prediction, this one-shot lower bound can be extended to the $n$-sample, $N$-subpopulation setting.

\begin{proposition}\label{claim:thresholds}
    There exists a metadistribution $q_{\T}$ over problem instances of (realizable) $d$-bit threshold learning such that any $\eps$-suboptimal algorithm receiving exactly one sample $X$ satisfies 
    \begin{align*}
        I(X;M\mid P) \ge (1- f(\eps) - o(1))\cdot H(X\mid P),
    \end{align*}
    where $f(\eps)\xrightarrow{\eps\to 0}0$, so the algorithm must memorize the whole sample as the error vanishes.
\end{proposition}

\begin{definition}[$q_{\T}$ Component Distribution]
    Draw threshold $c\in \bit{d}$ uniformly at random.
    To generate labeled data $(z,y)$, first pick a prefix length $\ell \in \{0,\ldots, d-1\}$ uniformly at random.
    Set
    \begin{align*}
        z(1:\ell) &= \begin{cases} c(1:\ell) & \text{w.p. $1/2$} \\
            \mathrm{Uniform}(\bit{\ell}) & \text{otherwise} \end{cases} \\
        z(\ell+1) &= 1 \\
        z(\ell+2: d) &= \vec{0}.
    \end{align*}
    Label according to the threshold: $y=0$ if $c\ge z$ and $y=1$ otherwise.
    \terminalbox
\end{definition}

The crucial feature of this distribution is, for data points which contain a length-$\ell$ prefix of $c$, the label is $1$ if and only if $c(\ell+1)=1$. 
This is clear with a small example, where $d=7$ and $\ell=4$:
\begin{align*}
    c &= \underline{1~0~0~1}~0~1~1 \\
    z &= \underline{1~0~0~1}~1~0~0.
\end{align*}
We have $c< z$ and thus $c(z)=c(\ell+1)=0$, capturing the ``next-bit'' behavior.

\begin{lemma}
    For any algorithm $A$ that is $\eps$-suboptimal for $q_{\T}$, there is an algorithm $A'$ for Singletons($1,q_{NSP}$), under noise parameter $\SCnoise=0$, that is $(4\eps+o(1))$-suboptimal.
    Furthermore, $I(X;A(X))=I(X;A'(X))$.
\end{lemma}
\begin{proof}
    Alice and Bob get two (noiseless) prefixes of $c$.
    Call them $z_A$ and $z_B$ and denote their lengths $\ell_A$ and $\ell_B$.
    Under algorithm $A'$, they create $\tilde{z}_A$ and $\tilde{z}_B$ by padding to length $d$ with ``$10\cdots 0$.''
    They then run $A$ on these padded inputs and return $A$'s output.
    The algorithms' (identical) outputs have the same mutual information with their inputs.

    To bound the error of $A'$, let $G$ be the good event that both samples (in the threshold problem) contain prefixes of $c$ (as opposed to randomly-drawn prefixes).
    Then, by construction,
    \begin{equation*}
        \Pr[A'\text{ errs on Singletons}(1,q_{NSP})]
            = \Pr[A\text{ errs on Singletons}(1,q_{\T})\mid G].
    \end{equation*}
    We expand over $G$ and $\bar{G}$:
    \begin{align*}
        \eps &= \Pr[A\text{ errs on Singletons}(1,q_{\T})] - \Pr[\Aopt\text{ errs on Singletons}(1,q_{\T})] \\
        &= \Pr[G]\left(\Pr[A\text{ errs on Singletons}(1,q_{\T})\mid G] - \Pr[\Aopt\text{ errs on Singletons}(1,q_{\T})\mid G]\right) \\
        &\quad + \Pr[\bar{G}]\left(\Pr[A\text{ errs on Singletons}(1,q_{\T})\mid \bar{G}] - \Pr[\Aopt\text{ errs on Singletons}(1,q_{\T})\mid \bar{G}]\right) \\
        &\ge \frac{1}{4} \left(\Pr[A\text{ errs on Singletons}(1,q_{\T})\mid G] - \Pr[\Aopt\text{ errs on Singletons}(1,q_{\T})\mid G]\right) - o(1) \\
        &\ge \frac{1}{4} \left(\Pr[A'\text{ errs on Singletons}(1,q_{NSP})] - \inf_{A''} \Pr[A''\text{ errs on Singletons}(1,q_{NSP})]\right) - o(1).
    \end{align*}
    To establish the inequalities, observe that $\Aopt$, with access to both inputs, may fail to learn if $G$ occurred only when the strings match by accident.
    The probability of this is no more than the probability that (1) both strings have length less than $\log d$, or (2) at least $\log d$ random bits match, both of which have probability $o(1)$.
    Rearranging finishes the proof.
\end{proof}

\begin{lemma}
    For Singletons($1,q_{\T}$), we have $I(X;P) = \frac{d+1}{4} + O(1)$.
\end{lemma}
\begin{proof}
    Let $L$ be the length of the data point and $B$ the indicator for ``$X$ was drawn independently of $c$.''
    We have
    \begin{align*}
        I(X;P) = I(X, L, B; P) + I(B;P\mid X) &= I(B;P) + I(L; P\mid B) + I(X;P\mid B,L) + O(1)\\
            &= 0 + 0 + I(X;P\mid B,L) + O(1)\\
            &= \frac{1}{2} I(X;P\mid B=1,L) + \frac{1}{2}I(X;P\mid B=0,L) + O(1)\\
            &= 0 + \frac{1}{2} \cdot \frac{d+1}{2}+ O(1),
    \end{align*}
    since, when $B=1$, the bits of $X$ are independent of the threshold and, when $B=1$ and $L=\ell$ for any $\ell$, the mutual information is exactly $\ell+1$.
\end{proof}

\begin{proof}[Proof of Claim~\ref{claim:thresholds}]
    Via the one-sample lower bound for Next-Symbol Prediction in Lemma~\ref{lemma:SC_one_shot_second} 
    and the reduction above, any $\eps$-suboptimal algorithm $A$ for Singletons($1,q_{\T}$) has
    \begin{align*}
        I(X; M) \ge \frac{d+1}{2} \cdot (1 - h(8\eps + o(1))).
    \end{align*}
    Reusing a calculation from Section~\ref{sec:blueprint} that uses the fact that $P\perp M\mid X$, we have
    \begin{align*}
        I(X; M\mid P) \ge I(X;M) - I(X;P) \ge \frac{d+1}{4}\cdot(1 - 2\cdot h(8\eps+o(1))).
    \end{align*}
\end{proof}

\subsection{Two-Length Next-Symbol Prediction}\label{sec:2_NSP}

We present a sequence prediction learning task which is less natural than Next-Symbol Prediction but, in contrast, allows Alice to send a significantly smaller message to Bob after she has received multiple samples.
We focus on the one-shot case.

\begin{definition}[$q_2$ Component Distribution]
    In Two-Length Next-Symbol Prediction, a problem instance is described by a reference string $c\in\{0,1\}^d$ and two indices $j,k\in\{0,\ldots,d-1\}$.
    The metadistribution $q_2$ samples $x, j$, and $k$ uniformly and independently.
    Given problem instance $P=(x,j,k)$, labeled data is generated i.i.d. in the following way.
    Flip a fair coin:
    \begin{itemize}
        \item If heads, return $(x_{1:j}, x_{j+1})$.
        \item If tails, return $(x_{1:k}, x_{k+1})$.
    \end{itemize}
    \terminalbox
\end{definition}

Although the problem instance needs $d + 2\log d$ bits to describe, Bob only needs to know $(j,k,x_{j+1},x_{k+1})$ to answer correctly.
This requires only $2(\log d + 1)$ bits.
If Alice recieves $n>1$ samples, she learns $(j,k,x_{j+1},x_{k+1})$ exactly as soon as she receives inputs of different lengths.
She may fail to learn the instance if $j=k$ or if all of her inputs have the same length, so this failure probability is bounded above by
\begin{align*}
    \Pr[\text{Alice fails to learn $(j,k,x_{j+1},x_{k+1})$}] \le \frac{1}{d} + \frac{1}{2^{n-1}}.
\end{align*}
Nevertheless, in the $n=1$ case, Alice will have to send almost all her input to compete with the optimal protocol.

\begin{theorem}
    Any learning algorithm for one-shot Two-Length Next-Symbol Prediction with error $\epsilon$ above optimal succeeds on (standard) Next-Symbol Prediction with error at most $2\epsilon$ above optimal.
\end{theorem}
\begin{proof}
    Let indicator random variables $C_A$ and $C_B$ denote, for Alice and Bob respectively, the result of the coin flip deciding if their inputs were to be of length $j$ or $k$.
    (Note that this coin is flipped even when $j=k$).
    Let $C=\mathrm{XOR}(C_A,C_B)$, so $C=1$ means one player received a length-$j$ string and the other received a length-$k$ string.
    The key observation is that conditioning on $C=1$ results in a distribution identical to that of noiseless one-shot Next-Symbol Prediction: Alice and Bob's inputs are prefixes of the same string with lengths chosen uniformly at random and Bob needs to answer the next bit.
    This any protocol with $\Pr[\text{$\alg$ correct}\mid C=1] = p$ also has accuracy $p$ on Next-Symbol Prediction.
    
    For protocol $\alg$ for Modified Next-Symbol Prediction, define
    \begin{align*}
        \epsilon' = \Pr[\text{$\alg$ errs}\mid C=1] - \Pr[\text{$\alg_{OPT}$ errs}\mid C=1],
    \end{align*}
    the error gap when applied to inputs from Next-Symbol Prediction.
    (Note that the optimal protocol for the modified task is also optimal for the standard task.)
    Then we have
    \begin{align*}
        \epsilon &= \Pr[\text{$\alg$ errs}] - \Pr[\text{$\alg_{OPT}$ errs}] \\ 
            &= \Pr[C=1] \left( \Pr[\text{$\alg$ errs}\mid C=1] - \Pr[\text{$\alg_{OPT}$ errs}]\mid C=1]\right) \\
            &\quad + \Pr[C=0] \left( \Pr[\text{$\alg$ errs}\mid C=0] - \Pr[\text{$\alg_{OPT}$ errs}]\mid C=0]\right) \\
            &\ge \frac{1}{2} \cdot \epsilon',
    \end{align*}
    where the inequality follows because the error of $\alg$ is at least that of $\alg_{OPT}$.
    This completes the proof.
\end{proof}

\begin{corollary}
    Any learning algorithm for one-shot Modified Next-Symbol Prediction with error $\epsilon$ above optimal satisfies
    \begin{align*}
        I(M;X) \ge \frac{d+1}{2}\left(1 - h(4\epsilon)\right).
    \end{align*}
\end{corollary}
\begin{proof}
    Lemma~\ref{lemma:SC_one_shot_second} asks for $\gamma$ such that
    \begin{align*}
        \Pr[\text{$\alg$ correct}] = \frac{1}{2} + OPT_{\mathrm{sing}} (1- \gamma).
    \end{align*}
    This implies $H(X\mid M) \le \frac{d+1}{2} h(\gamma/2)$.
    By the above theorem, on Next-Symbol Prediction $\alg$ has, converting from accuracy to error,
    \begin{align*}
        2\epsilon &\ge \Pr[\text{$\alg$ errs}] - \Pr[\text{$\alg_{OPT}$ errs}] \\
            &= \left(1 - \left(\frac{1}{2} + OPT_{\mathrm{sing}} (1-\gamma) \right)\right)
                - \left(1 - \left(\frac{1}{2} + OPT_{\mathrm{sing}} \right)\right) \\
            &= OPT_{\mathrm{sing}} \gamma \\
            &\ge \frac{\gamma}{4}.
    \end{align*}
    Letting $L$ be the length of Alice's sample (counting the label), we can compute the entropy:
    \begin{align*}
        H(X) = H(X, L) = H(L) + H(X\mid L) \ge \E_\ell[H(X\mid L=\ell)] = \frac{d+1}{2}.
    \end{align*}
    Combining the entropy lower bound and conditional entropy upper bound establishes the claim.
\end{proof}

\section{Differentially Private Algorithms Have High Error}\label{sec:DP_calc}

By definition, the output of a differentially private algorithm is not very sensitive to changes in the input data set.
In our setting, with data sets drawn i.i.d. from a fixed distribution, existing results from the differential privacy literature allow us to formalize this with an upper bound on mutual information:
\begin{proposition}\label{claim:DP_info_upbd}
    Fix distribution $P=p$ over examples in $\bit{d}$ and suppose the data $X$ is drawn from the product distribution $p^{\otimes n}$.
    Then, for any $(\alpha,\beta)$-differentially private algorithm $A$, 
    \begin{align*}
        I(X;A(X)\mid P=p) \le n \left(2\alpha\cdot  \frac{e^{2\alpha} -1}{e^{2\alpha}+1} + \beta d + h(\beta)\right).
    \end{align*}
\end{proposition}

Our lower bounds imply that, for small constant $\eps$, any $\eps$-suboptimal algorithm on the tasks we consider must have $I(X;A(X)\mid P)=\Omega(nd)$. 
This is inconsistent with the above bound, even if $\alpha = O(\log n)$ and $\beta$ is a sufficiently small constant.
Recall that an informal standard for ``meaningful privacy'' requires $\alpha=O(1)$ and $\delta \ll 1/n$.
Thus, even for unacceptably large values of the privacy parameters, differential privacy is incompatbile with low suboptimality.

The claim follows from well-known statements in the privacy literature, so we will only provide the high-level ideas.
We emphasize that such a statement holds only in the case of data from  product distributions.
The first step is to provide an upper bound on the mutual information of any $(\alpha,\beta)$-DP algorithm accepting a single input, i.e. a noninteractive LDP algorithm.
See e.g. Lemma 3.6 in \citep{cheu2020limits} for such a proof; one first shows that any one-sample $(\alpha,\beta)$-DP algorithm can be converted into a one-sample $(2\alpha,0)$-DP algorithm that is close in total variation distance.

The second step is to show that $n$ multiplied by this bound is itself an upper bound on the mutual information of any $n$-sample $(\alpha,\beta)$-DP algorithm.
This follows from a simple argument applying linearity of expectation and the fact that fixing the other $n-1$ inputs allows us to treat any $n$-sample algorithm as one-sample.


\section{One-Way Information Complexity of Gap-Hamming}\label{sec:one_way_GHP}

We provide a lower bound for the one-way information complexity of the Gap-Hamming problem.
This bound achieves the ``right'' constant, establishing that Alice must send $(1-o(1))\cdot d$ bits of information (out of the $d$ she receives) as her error vanishes.
This offers evidence for Conjecture~\ref{conj:HC_singletons_bound}, since the Singletons($k,q_{HC}$) task is a natural generalization of Gap-Hamming to multiple samples.
The two-way communication complexity lower bound of $\Omega(n)$ was first established in \citep{chakrabarti2012optimal}, with later papers providing simplified proofs (see~\citep{rao2020communication} for additional background.)
Notably for this work, \cite{hadar2019communication} offered an information-theoretic lower bound using a strong data-processing inequality.
A modified version of their proof yielded our proof of Lemma~\ref{lemma:HC_communication_bound}, but it is not clear that this approach can yield avoid losing a constant factor.
We introduce a simple and general technique for proving lower bounds on one-way communication over product distributions.
This technique generalizes and extends a proof in \citep{saglam2013communication}, allowing it to be applied to information complexity of general one-way communication problems, and showing that it can be used to avoid losing constant factors.
To the best of our knowledge, this is the first result achieving the $(1-o(1))$ factor for Gap-Hamming.

In the one-way Gap-Hamming problem, Alice gets $x\in\{0,1\}^d$ and Bob gets $y\in\{0,1\}^d$.
We will take these inputs to be uniform and independent distribution.\footnote{Since the inputs are independent and the communication is one-way, for any algorithm we have $I(M;X,Y) = I(Y;M) + I(X;M\mid Y) = I(X;M\mid Y)$, so the ``internal'' and ``external'' information complexities coincide.}
Bob's goal is to output, with probability at least $1-\epsilon$, the following partial function (using Hamming distance)
\begin{align*}
    GHP(x,y) = \begin{cases}
        1 & \text{if $d(x,y) \le \frac{d}{2} - c\sqrt{d}$} \\
        0 & \text{if $d(x,y) \ge \frac{d}{2} + c\sqrt{d}$} \\
        * & \text{otherwise} 
    \end{cases}
\end{align*}
This is a promise problem controlled by parameter $c$.
For any fixed $c$, with constant probability the promise holds.
We study what happens when Alice and Bob succeed with probability at least $1-\eps$ for a sufficiently small constant $\eps$.

\begin{theorem}\label{thm:GHP_oneshot_main}
    Let $\eps$ denote the probability that Alice and Bob make an error.
    For any fixed $c>0$, there exists a function $g_c(\eps,d)$ such that
    \begin{align*}
        IC_{\eps}^{\to}(GH_c) \ge (1-g_c(\eps,d))\cdot d,
    \end{align*}
    where $g_c(\eps,d)\to 0$ for any sequence of pairs $\{\eps_i,d_i\}_{i=1}^\infty$ such that $\eps_i\to 0$ and $d_i\to \infty$.
\end{theorem}

Theorem~\ref{thm:GHP_oneshot_main} follows, with a bit of calculation, from the following lemma.
To prove Lemma~\ref{lemma:one_shot_GHP}, we prove a function-agnostic statement about one-way information complexity over product distributions, then apply it to Gap-Hamming.

\begin{lemma}\label{lemma:one_shot_GHP}
    If Alice and Bob succeed with probability at least $1-\epsilon$ then, for sufficiently large $d$ and all $p\in (0,1)$,
    \begin{align*}
        I(M;X) \ge H(X) \cdot \left(1 - h(p) \right)\cdot \left( 1 - 2\epsilon/\alpha\right) - 1,
    \end{align*}
    where $\alpha$ can be bounded as follows: 
    \begin{align*}
        \alpha \ge \left(1 - 2 F_{\mc{N}(0,1)}\left(\frac{-c}{\sqrt{1-p}}\right) - \mc{O}\left(\frac{1}{\sqrt{d(1-p)}}\right)\right)
         \left(2 F_{\mc{N}(0,1)} \left(\frac{-3c}{\sqrt{p}}\right) - \mc{O}\left(\frac{1}{\sqrt{dp}}\right)\right).
    \end{align*}
    Here $F_{\mc{N}(0,1)}(\cdot)$ is the CDF of the standard normal distribution.
\end{lemma}

\paragraph{A General Approach to One-Way Information Complexity}
We will first prove the following lemma for lower bounding the one-way external information complexity of any communication task $f:\mc{X}\times\mc{X}\to\mc{Z}$ where Alice gets $X$, Bob gets $Y$, and $X\perp Y$.
To ``instantiate'' the lemma we need a notion of \textit{incompatibility} which applies to two of Alice' inputs.
For any definition of incompatibility, say $y$ \textit{distinguishes} $x, x'$ if $f(x,y)\neq f(x',y)$.
We will need a lower bound on distinguishing
\begin{align*}
    \Pr[\text{$y$ distinguishes $x,x'$}\mid \text{$x,x'$ incompatible}] \ge \alpha
\end{align*}
and an upper bound on the number of compatible inputs
\begin{align*}
     \max_x \left\lvert \left\{ x' : \text{$x,x'$ compatible}\right\}\right\rvert \le N_{\max}
\end{align*}

\begin{lemma}\label{lemma:product_bound}
    Suppose $X\perp Y$ and Alice and Bob succeed in computing $f(x,y)$ with probability $1-\epsilon$. 
    Fix a definition of incompatibility (which can depend on $\epsilon$).
    Then
    \begin{align*}
        I(X; M) \ge H(X) - \log N_{max} - 1 - \frac{2 \epsilon }{\alpha}(\log |\mathcal{X}| - \log N_{max}).
    \end{align*}
    If Alice's input is uniform and we can bound $\alpha \ge \Omega(\sqrt{\eps})$, this simplifies to 
    \begin{align*}
        I(M;X) \ge H(X)\cdot \left(1 - \frac{\log N_{\max}}{H(X)}\right) \cdot \left(1 - O\left(\sqrt{\eps}\right)\right) -  1.
    \end{align*}
\end{lemma}
\begin{proof}
    We perform the following thought experiment.
    Given Alice's message $m$, samples $x,x'\sim X\mid M=m$ independently.
    Let $D$ be the event that $y$ \textit{distinguishes} these two inputs: that $f(x,y)\neq f(x',y)$.
    Let $\Pi(m,y)$ denote Bob's portion of the protocol.\footnote{
    We assume $\Pi(m,y)$ is deterministic. This is without loss of generality, since given any $m,y$ pair there is a Bayes-optimal answer.
    Converting a randomized protocol to a Bayes-optimal one will not decrease the error and leaves $I(M;X)$ unaffected.}
    Observe that
    \begin{align}
        \Pr[D\mid M=m] &\le \Pr[\Pi(M,Y) \neq f(X,Y)\mid M=m] \nonumber \\
                    &\qquad + \Pr[\Pi(M,Y) \neq f(X',Y) \mid M=m] \nonumber \\        
            &= 2\Pr[\text{error}\mid M=m], \label{eq:distinguish_and_error}
    \end{align}
    since at least of one of $x, x'$ correspond to the wrong answer. 
    So the thought experiment allows us, with a lower bound on $\Pr[D\mid M=m]$, to show the protocol will have high error.
    
    Letting $C$ be the event that these two inputs are compatible, we have
    \begin{equation}
        \Pr[D\mid M=m] \ge \Pr[D\mid \bar{C}] \cdot\Pr[\bar{C}\mid M=m]. \label{eq:distinguish_compatible}
    \end{equation}
    The first term we bound using the problem itself and the definiton of incompatibility.
    Note that we've assumed $\Pr[D]$ is independent of $M$ conditioned on $C$. 
    The second term we bound via the entropy $H(X\mid M=m)$.
    
    Fix Alice's message $M=m$ to Bob.
    Using the fact that $X\perp X'\mid M$, we have (conflating events and indicator random variables)
    \begin{align*}
        H(X\mid M=m) = H(X\mid X', M=m) = H(X, C\mid X', M=m).
    \end{align*}
    Then, for any $x'$, we have via the chain rule for entropy that
    \begin{align*}
        H(X,C\mid X'=x',& M=m) \\
            &= H(X\mid C, X'=x', M=m) + H(C\mid X=x', M=m) \\
            &= \left(1 - \Pr[\bar{C}\mid X'=x', M=m'] \right)\cdot H(X\mid C, X'=x', M=m) \\
                &\qquad + \Pr[\bar{C}\mid X'=x', M=m'] \cdot H(X\mid C, X'=x', M=m) \\
                &\qquad + H(C\mid X=x', M=m) \\
            &\le \left(1 - \Pr[\bar{C}\mid X'=x', M=m'] \right)\cdot \log N_{\mathrm{max}} \\
                &\qquad + \Pr[\bar{C}\mid X'=x', M=m'] \cdot \log |\mathcal{X}| \\
                &\qquad + 1.
    \end{align*}    
    Take the expectation with respect to $X'\mid M=m$ and rearrange, getting 
    \begin{align*}
        \Pr[\bar{C}\mid M=m] \ge \frac{H(X\mid M=m) - \log N_{\mathrm{max}} - 1}{\log |\mathcal{X}| - \log N_{\mathrm{max}}}.
    \end{align*}
    Combining with \eqref{eq:distinguish_and_error} and \eqref{eq:distinguish_compatible}, we get 
    \begin{align*}
        2\Pr[\text{error}\mid M=m]\ge \Pr[D\mid \bar{C}] \cdot 
            \left(\frac{H(X\mid M=m) - \log N_{\mathrm{max}} - 1}{\log |\mathcal{X}| - \log N_{\mathrm{max}}}\right).
    \end{align*}
    Then, since both sides are linear, we take the expectation over $M$ and rearrange to bound the conditional entropy.
    \begin{align*}
        H(X\mid M) \le \log N_{max} + 1 + \frac{2 \epsilon }{\Pr[D\mid\bar{C}]}(\log |\mathcal{X}| - \log N_{max}).
    \end{align*}
    With $I(M;X) = H(X)-H(X\mid M)$, this turns into a lower bound on mutual information.
\end{proof}

\paragraph{Application to Gap Hamming}

We now return to Gap-Hamming. 
We define a notion of compatibility, upper bound $N_{\max}$, and lower bound $\alpha$.

\begin{proof}[Proof of Lemma \ref{lemma:one_shot_GHP}]
    To apply Lemma \ref{lemma:product_bound}, we must define a notion of compatibility for Alice's inputs, compute $N_{\max}$ to bound the number of compatible inputs, and lower bound the probability that a given $y$ distinguishes incompatible inputs.
    We will say 
    \begin{center}
        $x$ and $x'$ are compatible if $d_{H}(x,x') \le pd$.
    \end{center}
    Here. $p\in (0,1)$ is a parameter to be set later.
    For each $x$, then, the number of compatible $x'$ is the volume of the Hamming ball of radius $pd$, so we can bound $N_{\max} \le 2^{d \cdot h(p)}$, yielding
    \begin{align*}
        \frac{\log N_{\max}}{H(X)} \le h(p).
    \end{align*}

    We now lower-bound $\alpha$, the probability that $Y$ distinguishes $X$ and $X'$, assuming $X,X'$ agree in at most $d(1-p)$ locations.
    We need $Y$ to be close to one of $X,X'$ and far from the other.
    These Hamming distances are the result of independent fair coin flips, so we define the following random variables,
    \begin{align*}
        A\sim\mathrm{Bin}(d(1-p),1/2),\quad B\sim\mathrm{Bin}(pd,1/2)
    \end{align*}
    and note that
    \begin{align*}
        d(X,Y) &= A + B \\
        d(X',Y) &= A + pd - B.
    \end{align*}
    Define the following independent good events, introducing parameter $c'>0$: 
    \begin{gather*}
        G_A = \left\{\left|A - \frac{(1-p)d}{2}\right| \le c'\sqrt{d}\right\} \text{\ \ and\ \ }
        G_B = \left\{\left|B - \frac{pd}{2}\right| \ge (c+c')\sqrt{d}\right\}.
    \end{gather*}
    When both occur, $Y$ distinguishes $X, X'$.
    Assuming $B$ takes a value in its lower tail,
    \begin{align*}
        d(X,Y) &= A+B \le \left(\frac{(1-p)d}{2} + c'\sqrt{d}\right) + \left(\frac{pd}{2} - (c+c')\sqrt{d}\right) &= \frac{d}{2}-c\sqrt{d} \\
        d(X',Y) &= A+ pd-B \ge \left(\frac{(1-p)d}{2} - c'\sqrt{d}\right) + pd - \left(\frac{pd}{2} - (c+c')\sqrt{d} \right)&= \frac{d}{2}+c\sqrt{d}
    \end{align*}

    We lower bound $\Pr[G_A\cap G_B] = \Pr[G_A]\Pr[G_B]$ using the Berry-Ess{\'e}en Theorem, stated below.
    Define scaled
    \begin{align*}
        Z_A = \frac{A - d(1-p)/2}{\sqrt{d(1-p)/4}} \quad \text{ and }\quad Z_B = \frac{B - pd/2}{\sqrt{pd/4}}.
    \end{align*}
    Then
    \begin{align*}
        G_A &\Leftrightarrow \left|A - \frac{d(1-p)}{2}\right| \le c'\sqrt{d} \\
            &\Leftrightarrow  \left|Z_A \sqrt{d(1-p)/4} + \frac{d(1-p)}{2} - \frac{d(1-p)}{2}\right| \le c'\sqrt{d} \\
            &\Leftrightarrow |Z_A| \le \frac{2c'}{\sqrt{1-p}}.
    \end{align*}
    So
    \begin{align*}
        Pr[G_A] \ge 1 - 2 F_{\mc{N}(0,1)} \left(\frac{-2c'}{\sqrt{1-p}}\right) - \mc{O}\left(\frac{1}{\sqrt{d(1-p)}}\right)
    \end{align*}
    and, similarly,
    \begin{align*}
        Pr[G_B] \ge 2 F_{\mc{N}(0,1)} \left(\frac{-2(c+c')}{\sqrt{p}}\right) - \mc{O}\left(\frac{1}{\sqrt{dp}}\right).
    \end{align*}
    Taking $c'=\frac{c}{2}$ yields the statement.
\end{proof}

\begin{theorem}[Berry-Ess{\'e}en]\label{thm:BE}
    If $X\sim \mathrm{Bin}(n,1/2)$ and we have the scaled version $Z = \frac{X - n/2}{\sqrt{n/4}}$,
    then for all $a\in\mathbb{R}$
    \begin{align*}
        \left\lvert \Pr\left[Z \le a\right] - F_{\mc{N}(0,1)}(a) \right\rvert = \mathcal{O}\left(\frac{1}{\sqrt{n}}\right),
    \end{align*}
    where $F_{\mc{N}(0,1)}(a)$ is the CDF of the unit Gaussian evaluated at $a$.
\end{theorem}

\paragraph{Completing the Proof}
Lemma~\ref{lemma:one_shot_GHP} holds for all values of $p$.
We show how to set $p$ (as a function of $c, d$, and $\eps$) so that the information complexity is $(1-o(1))\cdot d$.

\begin{proof}[Proof of Lemma~\ref{lemma:one_shot_GHP}]
    Recall that we have
    \begin{align*}
        I(M;X) \ge H(X) \cdot \left(1 - h(p) \right)\cdot \left( 1 - 2\epsilon/\alpha\right) - 1,
    \end{align*}
    with $\alpha$: 
    \begin{align*}
        \alpha \ge \left(1 - 2 F_{\mc{N}(0,1)}\left(\frac{-c}{\sqrt{1-p}}\right) - \mc{O}\left(\frac{1}{\sqrt{d(1-p)}}\right)\right)
         \left(2 F_{\mc{N}(0,1)} \left(\frac{-3c}{\sqrt{p}}\right) - \mc{O}\left(\frac{1}{\sqrt{dp}}\right)\right).
    \end{align*}
    We wish to lower bound $I(X;M)$ as $\eps\to 0$ and $d\to \infty$.
    We will set $p= \frac{c_1}{\ln (c_2/\eps)}$ for some constants $c_1, c_2$ that will depend on $c$ but not $\eps$ or $d$.
    Thus $p\xrightarrow{\eps\to 0}0$, so the binary entropy function will go to zero.
    Furthermore, if $p\to 0$, then for any $c$ the first term $(1 - 2F(\cdot) - O(\cdot))$ will be $\Omega(1)$.
    Thus it remains to deal with the right term. 
    We will show
    \begin{align*}
        2 F_{\mc{N}(0,1)} \left(\frac{-3c}{\sqrt{p}}\right) - \mc{O}\left(\frac{1}{\sqrt{dp}}\right) = \Omega(\sqrt{\eps}).
    \end{align*}
    We can assume without loss of generality that $dp\to \infty$; if $\eps$ gets too small, we can set $p$ larger until $d$ ``catches up,'' since any algorithm erring with small probability also satisfies a looser probability bound.
    To finish, then, we just need a simple lower bound on the CDF.
    For sufficiently small $p$, we can use the rectangle of width $2$ whose left side sits at $\frac{-6c}{\sqrt{p}}$, so set
    \begin{align*}
        2 F_{\mc{N}(0,1)} \left(\frac{-3c}{\sqrt{p}}\right) \ge 2\sqrt{2}{\pi} e^{-9c^2/p}.
    \end{align*}
    There exist constants $c_1, c_2$ such that setting $p= \frac{c_1}{\ln (c_2/\eps)}$ will allow us to lower bound this term with $\alpha = \Omega(\sqrt{\eps})$, causing the $\frac{2\eps}{\alpha}$ term above above to vanish.
\end{proof}

\addcontentsline{toc}{section}{Acknowledgments}
\section*{Acknowledgments} We thank Ankit Garg for helpful conversations about the information complexity of the Gap-Hamming problem.


\addcontentsline{toc}{section}{References}
\bibliography{bibliography}

\end{document}